\def\year{2021}\relax
\documentclass[letterpaper]{article} 
\usepackage{aaai21}  
\usepackage{times}  
\usepackage{helvet} 
\usepackage{courier}  
\usepackage[hyphens]{url}  
\usepackage{graphicx} 
\usepackage{natbib}  
\usepackage{caption} 
\usepackage{microtype}
\usepackage{graphicx}
\usepackage{subcaption}
\usepackage{booktabs} 
\usepackage[export]{adjustbox}
\usepackage{amsfonts}
\usepackage{nicefrac}
\usepackage{amsmath}
\usepackage{mathtools}
\usepackage{amsthm}
\usepackage{commath}
\usepackage{bm}
\usepackage[textsize=tiny]{todonotes}
\usepackage{hyperref}
\usepackage{algorithm}
\usepackage{algorithmic}
\usepackage{multicol}
\usepackage[Symbol]{upgreek}

\usepackage{etoolbox}\AtBeginEnvironment{algorithmic}{\small}
\newcommand{\algrule}[1][.2pt]{\par\vskip.5\baselineskip\hrule height #1\par\vskip.5\baselineskip}

\usepackage{thm-restate}

\usepackage{xspace}
\DeclareRobustCommand{\eg}{e.g.,\@\xspace}
\DeclareRobustCommand{\ie}{i.e.,\@\xspace}
\DeclareRobustCommand{\wrt}{w.r.t.\@\xspace}
\newcommand{\compresslist}{%
    \setlength{\itemsep}{0pt}%
    \setlength{\parskip}{0pt}%
}
\usepackage{amssymb}
\newcommand{\vtheta}{\boldsymbol{\theta}}

\newcommand{\de}{\,\mathrm{d}}

\newcommand{\EV}{\mathbb{E}}

\DeclareMathOperator*{\Var}{Var \;}
\DeclareMathOperator*{\Cov}{Cov}

\title{Task-Agnostic Exploration via Policy Gradient \\ of a Non-Parametric State Entropy Estimate}
\author {
        Mirco Mutti\textsuperscript{\rm 1,2,}\thanks{Equal contribution.},
        Lorenzo Pratissoli\textsuperscript{\rm 1,}$^*$,
        \textnormal{and} Marcello Restelli\textsuperscript{\rm 1} \\
}
\affiliations {
    \textsuperscript{\rm 1} Politecnico di Milano, Milan, Italy \\
    \textsuperscript{\rm 2} Universit\`a di Bologna, Bologna, Italy \\
    mirco.mutti@polimi.it, lorenzo.pratissoli@mail.polimi.it, marcello.restelli@polimi.it \\
}

\begin{document}

\maketitle

\begin{abstract}
In a reward-free environment, what is a suitable intrinsic objective for an agent to pursue so that it can learn an optimal task-agnostic exploration policy?
In this paper, we argue that the entropy of the state distribution induced by finite-horizon trajectories is a sensible target.
Especially, we present a novel and practical policy-search algorithm, Maximum Entropy POLicy optimization (MEPOL), to learn a policy that maximizes a non-parametric, $k$-nearest neighbors estimate of the state distribution entropy. In contrast to known methods, MEPOL is completely model-free as it requires neither to estimate the state distribution of any policy nor to model transition dynamics.
Then, we empirically show that MEPOL allows learning a maximum-entropy exploration policy in high-dimensional, continuous-control domains, and how this policy facilitates learning meaningful reward-based tasks downstream.
\end{abstract}

\section{Introduction}
\label{sec:introduction}

In recent years, Reinforcement Learning (RL)~\cite{sutton2018reinforcement} has achieved outstanding results in remarkable tasks, such as Atari games~\cite{mnih2015human}, Go~\cite{silver2016mastering}, Dota 2~\cite{berner2019dota}, and dexterous manipulation~\cite{andrychowicz2020learning}.
To accomplish these feats, the learning process usually requires a considerable amount of human supervision, especially a hand-crafted reward function~\cite{hadfield2017inverse}, while the outcome rarely generalizes beyond a single task~\cite{cobbe2019quantifying}.
This barely mirrors human-like learning, which is far less dependent on exogenous guidance and exceptionally general. Notably, an infant would go through an intrinsically-driven, nearly exhaustive, exploration of the environment in an early stage, without knowing much about the tasks she/he will face. Still, this same unsupervised process will be consequential to solve those complex, externally-driven tasks, when they will eventually arise.
In this perspective, what is a suitable \emph{task-agnostic} exploration objective to set for the agent in an \emph{unsupervised} phase, so that the acquired knowledge would facilitate learning a variety of reward-based tasks afterwards?

Lately, several works have addressed this question in different directions.
In~\cite{bechtle2019meta, zheng2019intrinsic}, authors investigate how to embed task-agnostic knowledge into a transferable meta-reward function.
Other works~\cite{jin2020reward, tarbouriech2020active} consider the active estimation of the environment dynamics as an unsupervised objective.
Another promising approach, which is the one we focus in this paper, is to incorporate the unsupervised knowledge into a \emph{task-agnostic exploration policy}, obtained by maximizing some entropic measure over the state space~\cite{hazan19maxent, tarbouriech2019active, mutti2019ideal, lee2019efficient}.
Intuitively, an exploration policy might be easier to transfer than a transition model, which would be hardly robust to changes in the environment, and more ready to use than a meta-reward function, which would still require optimizing a policy as an intermediate step.
An ideal maximum-entropy policy, thus inducing a uniform distribution over states, is an extremely general starting point to solve any (unknown) subsequent goal-reaching task, as it minimizes the so-called worst-case regret~\citep[][Lemma 1]{gupta2018unsupervised}. In addition, by providing an efficient estimation of any, possibly sparse, reward function, it significantly reduces the burden on reward design.
In tabular settings, \citet{tarbouriech2019active, mutti2019ideal} propose theoretically-grounded methods for learning an exploration policy that maximizes the entropy of the asymptotic state distribution, while~\citet{mutti2019ideal} concurrently consider the minimization of the mixing time as a secondary objective.
In~\cite{hazan19maxent}, authors present a principled algorithm (MaxEnt) to optimize the entropy of the discounted state distribution of a tabular policy, and a theoretically-relaxed implementation to deal with function approximation. Similarly, \citet{lee2019efficient} design a method (SMM) to maximize the entropy of the finite-horizon state distribution. Both SMM and MaxEnt learn a maximum-entropy mixture of policies following this iterative procedure: first, they estimate the state distribution induced by the current mixture to define an intrinsic reward, then, they learn a policy that optimizes this reward to be added to the mixture.
Unfortunately, the literature approaches to state entropy maximization either consider impractical infinite-horizon settings~\cite{tarbouriech2019active, mutti2019ideal}, or output a mixture of policies that would be inadequate for non-episodic tasks~\cite{hazan19maxent, lee2019efficient}. In addition, they would still require a full model of the transition dynamics~\cite{tarbouriech2019active, mutti2019ideal}, or a state density estimation~\cite{hazan19maxent, lee2019efficient}, which hardly scale to complex domains.

In this paper, we present a novel policy-search algorithm~\cite{deisenroth2013survey}, to deal with task-agnostic exploration via state entropy maximization over a \emph{finite horizon}, which gracefully scales to continuous, high-dimensional domains. The algorithm, which we call \emph{Maximum Entropy POLicy optimization} (MEPOL), allows learning a \emph{single} maximum-entropy parameterized policy from mere interactions with the environment, combining non-parametric state entropy estimation and function approximation. It is completely \emph{model-free} as it requires neither to model the environment transition dynamics nor to directly estimate the state distribution of any policy.
The entropy of continuous distributions can be speculated by looking at how random samples drawn from them laid out over the support surface~\cite{beirlant1997nonparametric}. Intuitively, samples from a high entropy distribution would evenly cover the surface, while samples drawn from low entropy distributions would concentrate over narrow regions.
Backed by this intuition, MEPOL relies on a $k$-nearest neighbors entropy estimator~\cite{singh2003nearest} to asses the quality of a given policy from a batch of interactions. Hence, it searches for a policy that maximizes this entropy index within a parametric policy space. To do so, it combines ideas from two successful, state-of-the-art policy-search methods: TRPO~\cite{schulman2015trpo}, as it performs iterative optimizations of the entropy index within trust regions around the current policies, and POIS~\cite{metelli2018pois}, as these optimizations are performed offline via importance sampling. This recipe allows MEPOL to learn a maximum-entropy task-agnostic exploration policy while showing stable behavior during optimization.

The paper is organized as follows. First, we report the basic background (Section~\ref{sec:preliminaries}) and some relevant theoretical properties (Section~\ref{sec:theoretical_analysis}) that will be instrumental to subsequent sections. Then, we present the task-agnostic exploration objective (Section~\ref{sec:objective}), and a learning algorithm, MEPOL, to optimize it (Section~\ref{sec:algorithm}), which is empirically evaluated in Section~\ref{sec:experiments}. In Appendix~\ref{apx:related_work}, we discuss related work. The proofs of the theorems are reported in Appendix~\ref{apx:proof}.
The implementation of MEPOL can be found at \url{https://github.com/muttimirco/mepol}.

\section{Preliminaries}
\label{sec:preliminaries}

In this section, we report background and notation.

\paragraph{Markov Decision Processes} 

A discrete-time Markov Decision Process (MDP)~\cite{puterman2014markov} is defined by a tuple $\mathcal{M} = (\mathcal{S}, \mathcal{A}, P, R, d_0)$, where $\mathcal{S}$ and $\mathcal{A}$ are the state space and the action space respectively, $P(s' | s, a)$ is a Markovian transition model that defines the conditional probability of the next state $s'$ given the current state $s$ and action $a$, $R(s)$ is the expected immediate reward when arriving in state $s$, and $d_0$ is the initial state distribution. A trajectory $\tau \in \mathcal{T}$ is a sequence of state-action pairs $\tau = (s_0, a_0, s_1, a_1, \ldots)$.
A policy $\pi(a|s)$ defines the probability of taking action $a$ given the current state $s$. We denote by $\Pi$ the set of all stationary Markovian policies. A policy $\pi$ that interacts with an MDP, induces a $t$-step state distribution defined as (let $d_0^\pi = d_0$):
\begin{align*}
    d^\pi_t (s) &= Pr (s_t = s | \pi) = \int_{\mathcal{T}} Pr (\tau | \pi, s_t = s) \de \tau, \nonumber \\
    d^\pi_t (s) &= \int_{\mathcal{S}} d^\pi_{t - 1} (s') \int_{\mathcal{A}} \pi(a|s') P(s | s', a) \de a \de s',
    \label{eq:state_distribution}
\end{align*} 
for every $t > 0$. If the MDP is ergodic, it admits a unique steady-state distribution which is $\lim_{t \to \infty} d^\pi_t (s) = d^\pi (s)$. The mixing time $t_{\text{mix}}$ describes how fast the state distribution $d^\pi_t$ converges to its limit, given a mixing threshold $\epsilon$:
\begin{equation*}
    t_{\text{mix}} = \big\{  t \in \mathbb{N} : \sup_{s \in \mathcal{S}} \big| d^\pi_t (s) - d^\pi_{t-1} (s) \big| \leq \epsilon \big\}.
\end{equation*}

\paragraph{Differential Entropy}
Let $f(x)$ be a probability density function of a random vector $\bm{X}$ taking values in $\mathbb{R}^p$, then its differential entropy~\cite{shannon1948mathematical} is defined as:
\begin{equation*}
    H(f) = - \int f(x) \ln f(x) \de x.
\end{equation*}
When the distribution $f$ is not available, this quantity can be estimated given a realization of $\bm{X} = \{ x_i \}_{i = 1}^N$ \cite{beirlant1997nonparametric}. In particular, to deal with high-dimensional data, we can turn to non-parametric, $k$-Nearest Neighbors ($k$-NN) entropy estimators of the form~\cite{singh2003nearest}:
\begin{equation}
    \widehat{H}_k(f) = - \frac{1}{N} \sum_{i = 1}^N \ln \frac{k}{N V_i^k} + \ln k - \Psi (k), \label{eq:ons_estimator}
\end{equation}
where $\Psi$ is the digamma function, $\ln k - \Psi(k)$ is a bias correction term, $V_i^k$ is the volume of the hyper-sphere of radius $ R_i = | x_i - x_i^{k\text{-NN}} |$, which is the Euclidean distance between $x_i$ an its $k$-nearest neighbor $x_i^{k\text{-NN}}$, so that:
\begin{equation*}
    V_i^k = \frac{ \big| x_i - x_i^{k\text{-NN}} \big|^{p} \cdot \uppi^{\nicefrac{p}{2}}}{ \Gamma (\frac{p}{2} + 1)},
\end{equation*}
where $\Gamma$ is the gamma function, and $p$ the dimensions of $\bm{X}$.  The estimator~\eqref{eq:ons_estimator} is known to be asymptotically unbiased and consistent~\cite{singh2003nearest}. When the target distribution $f'$ differs from the sampling distribution $f$, we can provide an estimate of $H(f')$ by means of an Importance-Weighted (IW) $k$-NN estimator~\cite{ajgl2011differential}:
\begin{equation}
    \widehat{H}_k (f' | f) = - \sum_{i = 1}^N \frac{W_i}{k} \ln \frac{W_i}{V_i^k} + \ln k - \Psi (k), \label{eq:offs_estimator}
\end{equation}
where $W_i = \sum_{j \in \mathcal{N}_i^k} w_j$, such that $\mathcal{N}_i^k$ is the set of indices of the $k$-NN of $x_i$, and $w_j$ are the normalized importance weights of samples $x_j$, which are defined as:
\begin{equation*}
    w_j = \frac{ \nicefrac{ f'(x_j) }{ f(x_j) } }
    { \sum_{n = 1}^N \nicefrac{ f'(x_n) }{ f(x_n) } }.
\end{equation*}
As a by-product, we have access to a non-parametric IW $k$-NN estimate of the Kullback-Leibler (KL) divergence, given by~\cite{ajgl2011differential}:
\begin{equation}
    \widehat{D}_{KL} \big( f \big|\big| f') = \frac{1}{N} \sum_{i = 1}^N \ln \frac{ k \big/ N }{ \sum_{j \in \mathcal{N}_i^k} w_j}.
    \label{eq:kl_estimator}
\end{equation}
Note that, when $f' = f$, $w_j = \nicefrac{1}{N}$, the estimator~\eqref{eq:offs_estimator} is equivalent to~\eqref{eq:ons_estimator}, while $\widehat{D}_{KL} (f || f')$ is zero.

\section{Analysis of the Importance-Weighted Entropy Estimator}
\label{sec:theoretical_analysis}

In this section, we present a theoretical analysis over the quality of the estimation provided by~\eqref{eq:offs_estimator}. Especially, we provide a novel detailed proof of the bias, and a new characterization of its variance. Similarly as in~\citep[Theorem 8]{singh2003nearest} for the estimator~\eqref{eq:ons_estimator}, we can prove the following.
\begin{restatable}[]{thr}{estimatorBias} \label{thr:estimator_bias}
    \emph{\citep[Sec. 4.1]{ajgl2011differential}} Let $f$ be a sampling distribution, $f'$ a target distribution. The estimator $\widehat{H}_k ( f' | f )$ is asymptotically unbiased for any choice of $k$.
\end{restatable}
Therefore, given a sufficiently large batch of samples from an unknown distribution $f$, we can get an unbiased estimate of the entropy of any distribution $f'$, irrespective of the form of $f$ and $f'$. However, if the distance between the two grows large, a high variance might negatively affect the estimation.
\begin{restatable}[]{thr}{estimatorVariance} \label{thr:estimator_variance}
    Let $f$ be a sampling distribution, $f'$ a target distribution. The asymptotic variance of the estimator $\widehat{H}_k ( f' | f )$ is given by:
    \begin{align*}
        \lim_{N \to \infty} &\Var_{x \sim f} \big[ \widehat{H}_k (f'| f) \big]
        = \frac{1}{N} \bigg( \Var_{x \sim f} \big[ \overline{w} \ln \overline{w} \big] \\
        &+ \Var_{x \sim f} \big[ \overline{w} \ln R^p  \big]
        + \big( \ln C \big)^2 \Var_{x \sim f} \big[ \overline{w} \big] \bigg),
    \end{align*}
    where $\overline{w} = \frac{f'(x)}{f(x)}$, and $C = \frac{N \uppi^{\nicefrac{p}{2}}}{k \Gamma (\nicefrac{p}{2} + 1)}$ is a constant.
\end{restatable}

\section{A Task-Agnostic Exploration Objective}
\label{sec:objective}

In this section, we define a learning objective for task-agnostic exploration, which is a fully unsupervised phase that potentially precedes a set of diverse goal-based RL phases.
First, we make a common regularity assumption on the class of the considered MDPs, which allows us to exclude the presence of unsafe behaviors or dangerous states.
\begin{restatable}[]{ass}{ergodicity}
	For any policy $\pi \in \Pi$, the corresponding Markov chain $P^\pi$ is ergodic.
\end{restatable}
Then, following a common thread in maximum-entropy exploration~\cite{hazan19maxent, tarbouriech2019active, mutti2019ideal}, and particularly~\cite{lee2019efficient}, which focuses on a finite-horizon setting as we do, we define the \emph{task-agnostic exploration} problem:
\begin{align}
    &\underset{\pi \in \Pi}{\text{maximize}}\;\; \mathcal{F}_{\text{TAE}}(\pi) = H \bigg( \frac{1}{T} \sum_{t = 1}^{T} d^\pi_t \bigg),
    \label{eq:tae_entropy}
\end{align}
where $ \bar{d}_T = \frac{1}{T} \sum_{t = 1}^{T} d^\pi_t $ is the \emph{average state distribution}. An optimal policy \wrt this objective favors a maximal coverage of the state space into the finite-horizon $T$, irrespective of the state-visitation order. Notably, the exploration horizon $T$ has not to be intended as a given trajectory length, but rather as a parameter of the unsupervised exploration phase which allows to tradeoff exploration quality (\ie state-space coverage) with exploration efficiency (\ie mixing properties).

As the thoughtful reader might realize, optimizing Objective~\eqref{eq:tae_entropy} is not an easy task. Known approaches would require either to estimate the transition model in order to obtain average state distributions~\cite{tarbouriech2019active, mutti2019ideal}, or to directly estimate these distributions through a density model~\cite{hazan19maxent, lee2019efficient}. In contrast to the literature, we turn to non-parametric entropy estimation without explicit state distributions modeling, deriving a more practical policy-search approach that we present in the following section.

\section{The Algorithm}
\label{sec:algorithm}

In this section, we present a model-free policy-search algorithm, \emph{Maximum Entropy POLicy optimization} (MEPOL), to deal with the task-agnostic exploration problem~\eqref{eq:tae_entropy} in continuous, high-dimensional domains. MEPOL searches for a policy that maximizes the performance index $\widehat{H}_k ( \bar{d}_T ( \vtheta ))$ within a parametric space of stochastic differentiable policies $\Pi_\Theta = \{ \pi_{\vtheta} : \vtheta \in \Theta \subseteq \mathbb{R}^q \}$. The performance index is given by the non-parametric entropy estimator~\eqref{eq:ons_estimator} where we replace $f$ with the average state distribution $\bar{d}_T (\cdot | \pi_{\vtheta} ) = \bar{d}_T (\vtheta )$. The approach combines ideas from two successful policy-search algorithms, TRPO~\cite{schulman2015trpo} and POIS~\cite{metelli2018pois}, as it is reported in the following paragraphs. Algorithm~\ref{alg:MEPOL} provides the pseudocode for MEPOL.
\begin{algorithm}[t]
    \caption{MEPOL}
    \label{alg:MEPOL}
    \begin{algorithmic}[H]
        \STATE \textbf{Input}: exploration horizon $T$, sample-size $N$, trust-region threshold $\delta$, learning rate $\alpha$, nearest neighbors $k$
        \STATE initialize $\vtheta$
        \FOR{epoch = $1, 2, \ldots, $ until convergence}
            \STATE draw a batch of $\lceil \frac{N}{T} \rceil$ trajectories of length $T$ with $\pi_{\vtheta}$
            \STATE build a dataset of particles $\mathcal{D}_\tau = \{ (\tau_i^t, s_i) \}_{i = 1}^N$
            \STATE $\vtheta' = \emph{IS-Optimizer}(\mathcal{D}_\tau, \vtheta) $
            \STATE $\vtheta \gets \vtheta'$
        \ENDFOR
        \STATE \textbf{Output}: task-agnostic exploration policy $\pi_{\vtheta}$ 
    \end{algorithmic}
    \vspace{-3pt}
    \algrule[0.75pt]
    \begin{flushleft}
        \vspace{-4pt}
        \normalsize{IS-Optimizer}
        \vspace{-9pt}
    \end{flushleft}
    \algrule[0.4pt]
    \vspace{-5pt}
    \begin{algorithmic}[H]
        \STATE \textbf{Input}: dataset of particles $\mathcal{D}_\tau$, sampling parameters $\vtheta$
        \STATE initialize $h = 0$ and $\vtheta_h = \vtheta$
        \WHILE{$D_{KL}( \bar{d}_T (\vtheta_0)  || \bar{d}_T (\vtheta_h) ) \leq \delta$}
            \STATE compute a gradient step:
            \STATE $ \quad \vtheta_{h + 1} = \vtheta_{h} + \alpha \nabla_{\vtheta_{h}} \widehat{H}_k \big( \bar{d}_T (\vtheta_{h}) | \bar{d}_T ( \vtheta_0 ) \big)$
            \STATE $h \gets h + 1$
        \ENDWHILE
        \STATE \textbf{Output}: parameters $\vtheta_h$ 
    \end{algorithmic}
\end{algorithm}

\paragraph{Trust-Region Entropy Maximization}
The algorithm is designed as a sequence of entropy index maximizations, called epochs, within a trust-region around the current policy $\pi_{\vtheta}$~\cite{schulman2015trpo}. First, we select an exploration horizon $T$ and an estimator parameter $k \in \mathbb{N}$. Then, at each epoch, a batch of trajectories of length $T$ is sampled from the environment with $\pi_{\vtheta}$, so as to take a total of $N$ samples. By considering each state encountered in these trajectories as an unweighted particle, we have $\mathcal{D} = \{ s_i \}_{i = 1}^{N} $ where $ s_i \sim \bar{d}_T ( \vtheta )$. Then, given a trust-region threshold $\delta$, we aim to solve the following optimization problem:
\begin{equation}
\begin{aligned}
    \underset{\vtheta' \in \Theta}{\text{maximize}}& \quad
    \widehat{H}_k \big( \bar{d}_{T} (\vtheta')\big) \\
    \text{subject to}& \quad D_{KL} \big( \bar{d}_T (\vtheta) \big|\big| \bar{d}_T (\vtheta') \big) \leq \delta.
    \label{eq:tr_maximization}
\end{aligned}
\end{equation}
The idea is to optimize Problem~\eqref{eq:tr_maximization} via Importance Sampling (IS)~\cite{owen2013monte}, in a fully off-policy manner partially inspired by~\cite{metelli2018pois}, exploiting the IW entropy estimator~\eqref{eq:offs_estimator} to calculate the objective and the KL estimator~\eqref{eq:kl_estimator} to compute the trust-region constraint. We detail the off-policy optimization in the following paragraph.

\paragraph{Importance Sampling Optimization}
We first expand the set of particles $\mathcal{D}$ by introducing $\mathcal{D}_\tau = \{ (\tau_i^t, s_i) \}_{i = 1}^N$, where $\tau_i^t = (s_i^0, \ldots, s_i^t = s_i)$ is the portion of the trajectory that leads to state $s_i$. In this way, for any policy $\pi_{\vtheta'}$, we can associate to each particle its normalized importance weight:
\begin{align*}
    &\overline{w}_i = \frac{Pr(\tau_i^t | \pi_{\vtheta'})}{Pr(\tau_i^t | \pi_{\vtheta})} = \prod_{z=0}^t \frac{\pi_{\vtheta'} (a_i^z | s_i^z)}{\pi_{\vtheta} (a_i^z | s_i^z)},
    &w_i = \frac{ \overline{w}_i }
    {\sum_{n = 0}^N \overline{w}_n}.
\end{align*}
Then, having set a constant learning rate $\alpha$ and the initial parameters $ \vtheta_0 = \vtheta $, we consider a gradient ascent optimization of the IW entropy estimator~\eqref{eq:offs_estimator},
\begin{equation}
    \vtheta_{h + 1} = \vtheta_{h} + \alpha \nabla_{\vtheta_{h}}
    \widehat{H}_k \big( \bar{d}_T (\vtheta_{h}) | \bar{d}_T ( \vtheta_0 ) \big),
    \label{eq:policy_update}
\end{equation}
 until the trust-region boundary is reached, \ie when it holds:
\begin{equation*}
    \widehat{D}_{KL} \big( \bar{d}_T (\vtheta_0) \big|\big| \bar{d}_T (\vtheta_{h + 1}) \big) > \delta.
\end{equation*}
The following theorem provides the expression for the gradient of the IW entropy estimator in Equation~\eqref{eq:policy_update}.
\begin{restatable}[]{thr}{estimatorGradient} \label{thr:estimator_gradient}
    Let $\pi_{\vtheta}$ be the current policy and $\pi_{\vtheta'}$ a target policy. The gradient of the IW estimator $\widehat{H}_k ( \bar{d}_T (\vtheta') | \bar{d}_T (\vtheta) )$ \wrt $\vtheta'$ is given by:
    \begin{align*}
        \nabla_{\vtheta'} \widehat{H}_k &( \bar{d}_T (\vtheta') | \bar{d}_T (\vtheta) ) = - \sum_{i = 0}^{N} \frac{  \nabla_{\vtheta'} W_i }{ k } \bigg( V_i^k + \ln \frac{ W_i }{ V_i^k } \bigg),
    \end{align*}
    where:
    \begin{align*}
        &\nabla_{\vtheta'} W_i = \sum_{j \in \mathcal{N}_i^k} w_j \times \bigg( \sum_{z = 0}^{t} \nabla_{\vtheta'} \ln \pi_{\vtheta'} (a_j^z | s_j^z) \\
        &\;\; - \frac{ \sum_{n = 1}^N \prod_{z=0}^t \frac{ \pi_{\vtheta'} ( a_n^z | s_n^z ) }{ \pi_{\vtheta} ( a_n^z | s_n^z ) } \sum_{z = 0}^{t} \nabla_{\vtheta'} \ln \pi_{\vtheta'} (a_n^z | s_n^z)  } { \sum_{n = 1}^N \prod_{z=0}^t \frac{ \pi_{\vtheta'} ( a_n^z | s_n^z ) }{ \pi_{\vtheta} ( a_n^z | s_n^z ) } } \bigg).
    \end{align*}
\end{restatable}

\section{Empirical Analysis}
\label{sec:experiments}

In this section, we present a comprehensive empirical analysis, which is organized as follows:
\begin{enumerate} \compresslist
\item[\ref{sec:exp_tae})] We illustrate that MEPOL allows learning a maximum-entropy policy in a variety of continuous domains, outperforming the current state of the art (MaxEnt);
\item[\ref{sec:exp_mixing})] We illustrate how the exploration horizon $T$, over which the policy is optimized, maximally impacts the trade-off between state entropy and mixing time;
\item[\ref{sec:exp_goal})] We reveal the significant benefit of initializing an RL algorithm (TRPO) with a MEPOL policy to solve numerous challenging continuous control tasks.
\end{enumerate}
A thorough description of the experimental set-up, additional results, and visualizations are provided in Appendix~\ref{apx:experiments}.

\subsection{Task-Agnostic Exploration Learning}
\label{sec:exp_tae}

\begin{figure*}[t]
	\begin{subfigure}[t]{0.28\textwidth}
    \centering
        \includegraphics[valign=t]{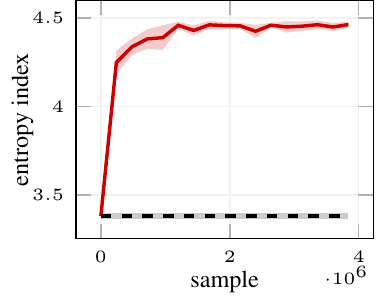}
        \vspace{-0.25cm}
        \caption{GridWorld}
        \label{fig:entropy_grid_world}
    \end{subfigure}
    \hspace{-0.2cm}
    \begin{subfigure}[t]{0.28\textwidth}
    \centering
        \includegraphics[valign=t]{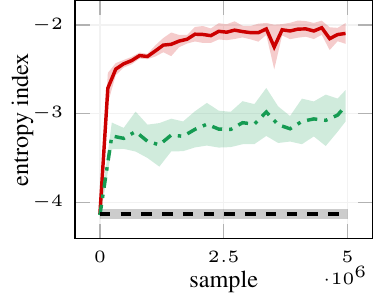}
        \vspace{-0.25cm}
        \caption{MountainCar}
        \label{fig:entropy_mountain_car}
    \end{subfigure}
    \begin{subfigure}[t]{0.42\textwidth}
    \centering
        \includegraphics[scale=0.175, valign=t]{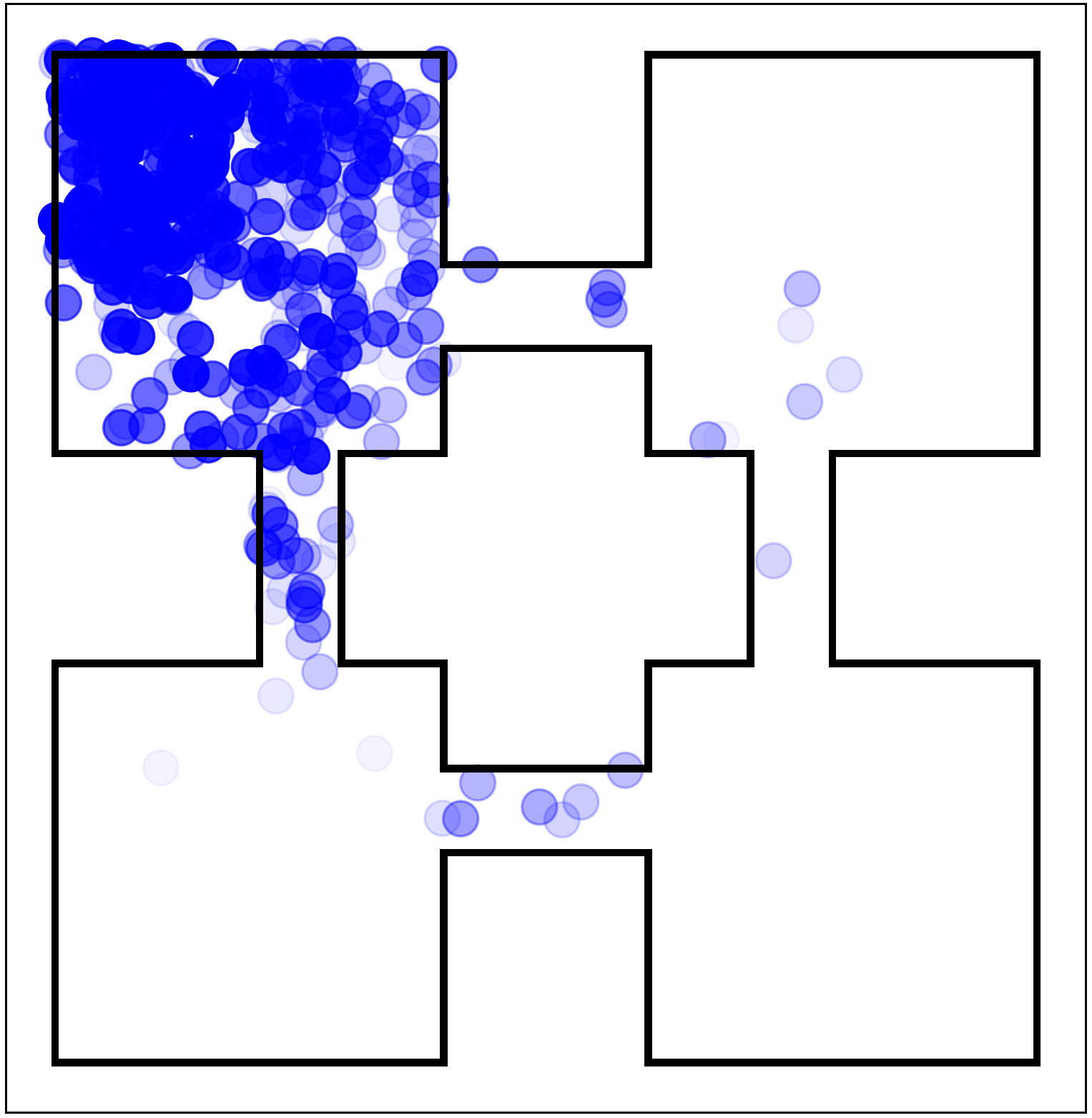}
        \hspace{0.3cm}
        \includegraphics[scale=0.175, valign=t]{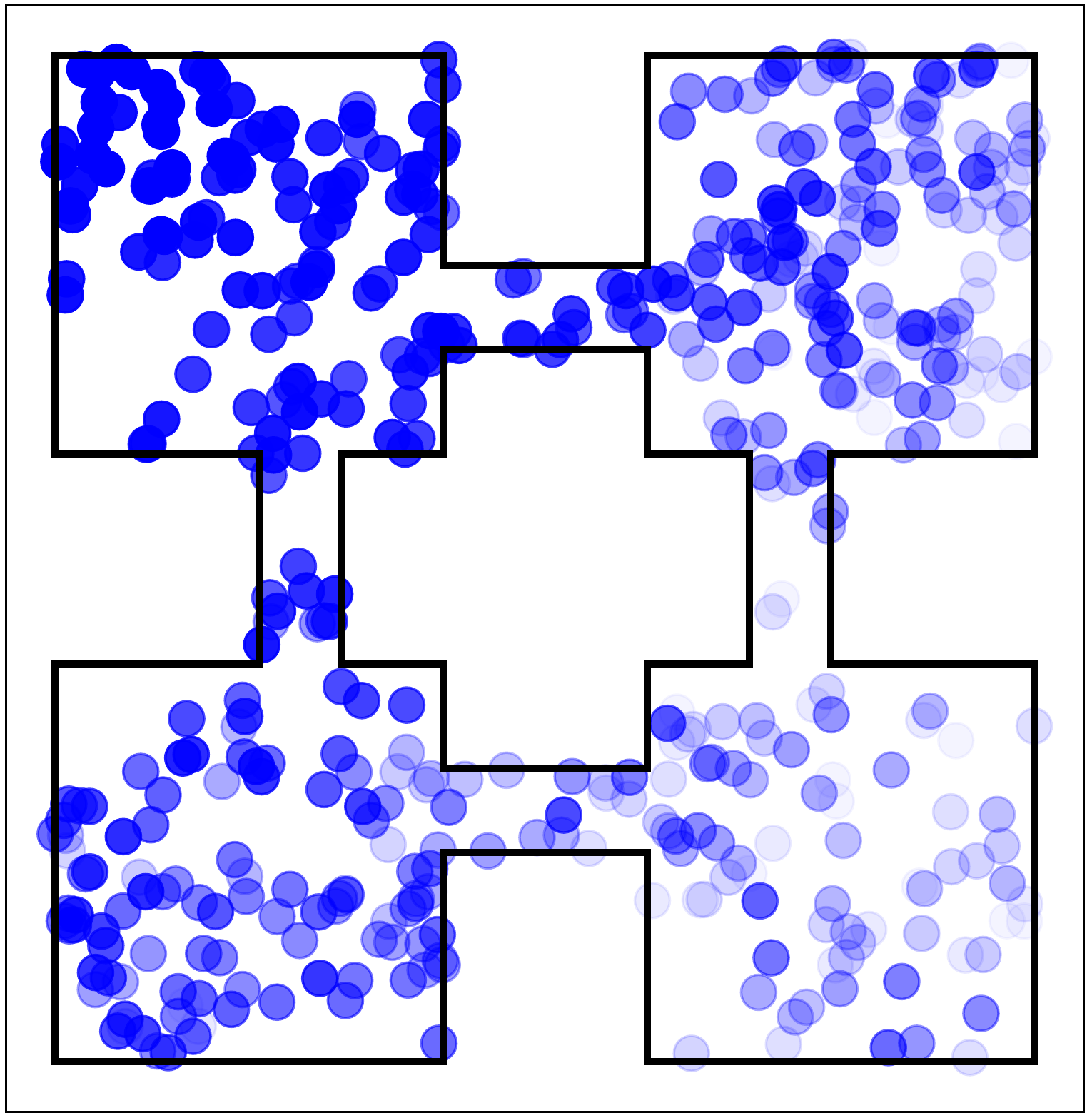}
        \vspace{0.02cm}
        \caption{GridWorld Visualization}
        \label{fig:illustrative_domains}
    \end{subfigure}

	\vspace{0.1cm}
	\begin{subfigure}[t]{0.33\textwidth}
    \centering
		\includegraphics[valign=t]{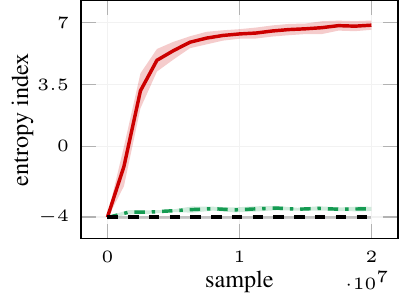}
		\vspace{-0.25cm}
        \caption{Ant}
        \label{fig:entropy_ant_torso}
    \end{subfigure}
    \begin{subfigure}[t]{0.32\textwidth}
    \centering
        \includegraphics[valign=t]{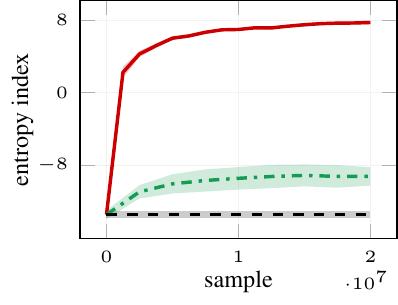}
        \vspace{-0.25cm}
        \caption{Humanoid}
        \label{fig:entropy_humanoid}
    \end{subfigure}
    \begin{subfigure}[t]{0.33\textwidth}
    \centering
        \includegraphics[valign=t]{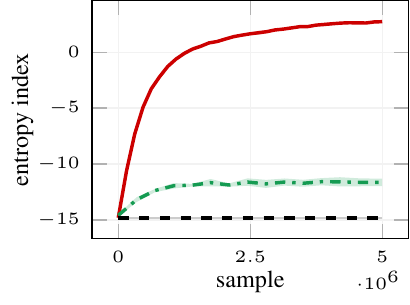}
        \vspace{-0.25cm}
        \caption{HandReach}
        \label{fig:entropy_hand_reach}
    \end{subfigure}

    \vspace{-0.05cm}
	\centering \includegraphics[scale=0.98]{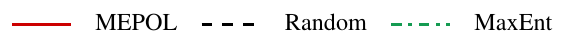}
	\vspace{-0.2cm}
    \caption{Comparison of the entropy index as a function of training samples achieved by MEPOL, MaxEnt, and a random policy. (95\% c.i. over 8 runs. MEPOL: $k$: 4 (c, d, e, f), 50 (b); $T$: 400 (c), 500 (d, e, f), 1200 (b). MaxEnt epochs: 20 (c), 30 (d, e, f)).}
    \label{fig:entropy_maximization}
\end{figure*}
In this section, we consider the ability of MEPOL to learn a task-agnostic exploration policy according to the proposed objective~\eqref{eq:tae_entropy}. Such a policy is evaluated in terms of its induced entropy value $\widehat{H}_k (\bar{d}_T (\vtheta))$, which we henceforth refer as \emph{entropy index}. We chose $k$ to optimize the performance of the estimator, albeit experiencing little to none sensitivity to this parameter (Appendix~\ref{apx:sensitivity}). In any considered domain, we picked a specific $T$ according to the time horizon we aimed to test in the subsequent goal-based setting (Section~\ref{sec:exp_goal}). This choice is not relevant in the policy optimization process, while we discuss how it affects the properties of the optimal policy in the next section. Note that, in all the experiments, we adopt a neural network to represent the parametric policy $\pi_{\vtheta}$ (see Appendix~\ref{apx:exp_setup_policy}). We compare our algorithm with MaxEnt~\cite{hazan19maxent}. To this end, we considered their practical implementation\footnote{\url{https://github.com/abbyvansoest/maxent/tree/master/humanoid}} of the algorithm to deal with continuous, non-discretized domains (see Appendix~\ref{apx:max_ent} for further details). Note that MaxEnt learns a mixture of policies rather than a single policy. To measure its entropy index, we stick with the original implementation by generating a batch as follows: for each step of a trajectory, we sample a policy from the mixture and we take an action with it. This is not our design choice, while we found that using the mixture in the usual way  leads to inferior performance anyway. We also investigated SMM~\cite{lee2019efficient} as a potential comparison. We do not report its results here for two reasons: we cannot achieve significant performance \wrt the random baseline, the difference with MaxEnt is merely in the implementation.

First, we evaluate task-agnostic exploration learning over two continuous illustrative domains: GridWorld (2D states, 2D actions) and MountainCar (2D, 1D). In these two domains, MEPOL successfully learns a policy that evenly covers the state space in a single batch of trajectories (state-visitation heatmaps are reported in Appendix~\ref{apx:heatmaps}), while showcasing minimal variance across different runs (Figure~\ref{fig:entropy_grid_world},~\ref{fig:entropy_mountain_car}). Notably, it significantly outperforms MaxEnt in the MountainCar domain.\footnote{We avoid the comparison in GridWorld, since the environment resulted particularly averse to MaxEnt.}
Additionally, In Figure~\ref{fig:illustrative_domains} we show how a batch of samples drawn with a random policy (left) compares to one drawn with an optimal policy (right, the color fades with the time step).
Then, we consider a set of continuous control, high-dimensional environments from the Mujoco suite~\cite{todorov2012mujoco}: Ant (29D, 8D), Humanoid (47D, 20D), HandReach (63D, 20D). While we learn a policy that maps full state representations to actions, we maximize the entropy index over a subset of the state space dimensions: 7D for Ant (3D position and 4D torso orientation), 24D for Humanoid (3D position, 4D body orientation, and all the joint angles), 24D for HandReach (full set of joint angles).
As we report in Figure~\ref{fig:entropy_ant_torso},~\ref{fig:entropy_humanoid},~\ref{fig:entropy_hand_reach}, MEPOL is able to learn policies with striking entropy values in all the environments. As a by-product, it unlocks several meaningful high-level skills during the process, such as jumping, rotating, navigation (Ant), crawling, standing up (Humanoid), and basic coordination (Humanoid, HandReach). Most importantly, the learning process is not negatively affected by the increasing number of dimensions, which is, instead, a well-known weakness of approaches based on explicit density estimation to compute the entropy~\cite{beirlant1997nonparametric}. This issue is documented by the poor results of MaxEnt, which struggles to match the performance of MEPOL in the considered domains, as it prematurely converges to a low-entropy mixture.

\paragraph{Scalability}
As we detail above, in the experiments over continuous control domains we do not maximize the entropy over the full state representation.
Note that this selection of features is not dictated by the inability of MEPOL to cope with even more dimensions, but to obtain reliable and visually interpretable behaviors (see Appendix~\ref{apx:tae_setup} for further details). 
To prove this point we conduct an additional experiment over a massively high-dimensional GirdWorld domain (200D, 200D). As we report in Figure~\ref{fig:200D_grid_world}, even in this setting MEPOL handily learns a policy  to maximize the entropy index.

\paragraph{On MaxEnt Results}
One might realize that the performance reported for MaxEnt appears to be much lower than the one presented in~\cite{hazan19maxent}. In this regard, some aspects need to be considered. First, their objective is different, as they focus on the entropy of discounted stationary distributions instead of $\bar{d}_T$. However, in the practical implementation, they consider undiscounted, finite-length trajectories as we do. Secondly, their results are computed over all samples collected during the learning process, while we measure the entropy over a single batch. Lastly, one could argue that an evaluation over the same measure ($k$-NN entropy estimate) that our method explicitly optimize is unfair. Nevertheless, even evaluating over the entropy of the 2D-discretized state space, which is the measure considered in~\cite{hazan19maxent}, leads to similar results (as reported in Figure~\ref{fig:discrete_entropy}).

\begin{figure*}[ht!]
	\begin{subfigure}[t]{0.65\textwidth}
	\vspace{-0.2cm}
	\centering
	 	\begin{tabular}[t]{lccccc}
        		\toprule
        		&MountainCar &Ant &Humanoid\\
        		\midrule
        		samples &$5\cdot10^6$ &$2\cdot10^7$  &$2\cdot10^7$ \\
        		\midrule
        		MEPOL &\textbf{4.31 $\pm$ 0.04} &\textbf{3.67 $\pm$ 0.05}  &\textbf{1.92 $\pm$ 0.08}\\
        		MaxEnt & 3.36 $\pm$ 0.4  &1.92 $\pm$ 0.05 &0.96 $\pm$ 0.06\\
        		Random & 1.98 $\pm$ 0.05  &1.86 $\pm$ 0.06 &0.84 $\pm$ 0.04\\
        		\bottomrule
    		\end{tabular}
    		\caption{Comparison of the entropy over the 2D-discretized states achieved by MEPOL, MaxEnt, and a random policy (95\% c.i. over 8 runs).}
    		\label{fig:discrete_entropy}
    \end{subfigure}
    \hspace{0.45cm}
    \begin{subfigure}[t]{0.3\textwidth}
		\includegraphics[valign=t]{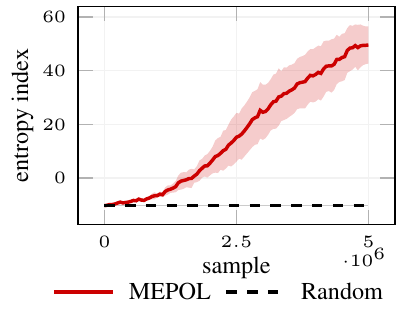}
 		\vspace{-0.25cm}
    		\caption{200D-GridWorld}
    		\label{fig:200D_grid_world}
    \end{subfigure}

    \vspace{0.4cm}
    \begin{subfigure}[t]{0.66\textwidth}
    \centering \includegraphics[valign=t]{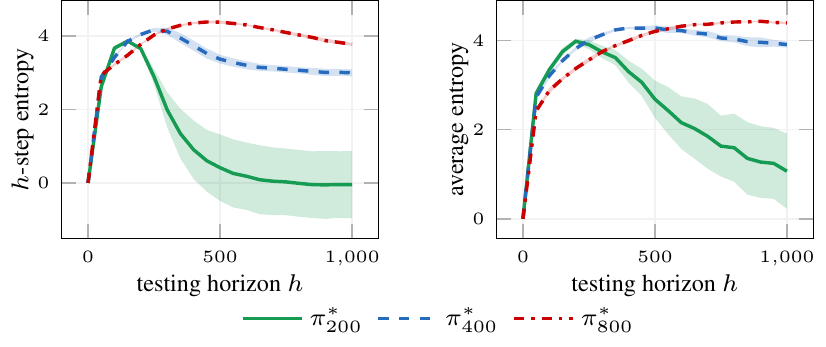}
    		\vspace{-0.25cm}
		\caption{GridWorld}
		\label{fig:tradeoff_grid}
	\end{subfigure}
 	\begin{subfigure}[t]{0.33\textwidth}
 		\centering \includegraphics[valign=t]{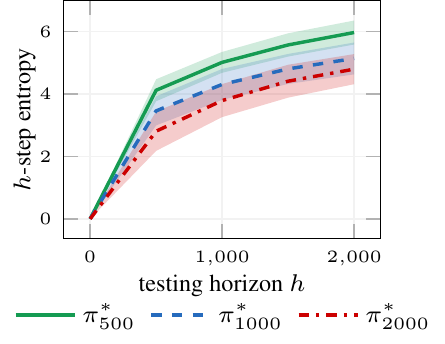}
 		\vspace{-0.25cm}
 		\caption{Humanoid}
 		\label{fig:tradeoff_humanoid}
 	\end{subfigure}
 	\vspace{-0.2cm}
\caption{Comparison of the entropy index over an extended (200D, 200D) GridWorld domain (b).
Comparison of the $h$-step entropy ($H (d^\pi_h)$) and average entropy ($H(\bar{d}_h)$) achieved by a set of policies trained over different horizons $T$ as a function of the testing horizon $h$ (c, d). (95\% c.i. over 8 runs).}
\end{figure*}

\subsection{Impact of the Exploration Horizon Parameter}
\label{sec:exp_mixing}

In this section, we discuss how choosing an exploration horizon $T$ affects the properties of the learned policy. First, it is useful to distinguish between a \emph{training horizon} $T$, which is an input parameter to MEPOL, and a \emph{testing horizon} $h$ on which the policy is evaluated. Especially, it is of particular interest to consider how an exploratory policy trained over $T$-steps fares in exploring the environment for a mismatching number of steps $h$. To this end, we carried out a set of experiments in the aforementioned GridWorld and Humanoid domains. We denote by $\pi_T^*$ a policy obtained by executing MEPOL with a training horizon $T$ and we consider the entropy of the $h$-step state distribution induced by $\pi_T^*$. Figure~\ref{fig:tradeoff_grid} (left), referring to the GridWorld experiment, shows that a policy trained over a shorter $T$ might hit a peak in the entropy measure earlier (fast mixing), but other policies achieve higher entropy values at their optimum (highly exploring).\footnote{The trade-off between entropy and mixing time has been substantiated for steady-state distributions in~\cite{mutti2019ideal}.} It is worth noting that the policy trained over $200$-steps becomes overzealous when the testing horizon extends to higher values, while derailing towards a poor $h$-step entropy. In such a short horizon, the learned policy cannot evenly cover the four rooms and it overfits over easy-to-reach locations. Unsurprisingly, also the average state entropy over $h$-steps ($\bar{d}_{h}$), which is the actual objective we aim to maximize in task-agnostic exploration, is negatively affected, as we report in Figure~\ref{fig:tradeoff_grid} (right).
This result points out the importance of properly choosing the training horizon in accordance with the downstream-task horizon the policy will eventually face. However, in other cases a policy learned over $T$-steps might gracefully generalize to longer horizons, as confirmed by the Humanoid experiment (Figure~\ref{fig:tradeoff_humanoid}). The environment is free of obstacles that can limit the agent's motion, so there is no incentive to overfit an exploration behavior over a shorter $T$.

\subsection{Goal-Based Reinforcement Learning}
\label{sec:exp_goal}

\begin{figure*}[t]
	\begin{subfigure}[t]{0.99\textwidth}
    \begin{center}
    		\hspace{0.5cm}
        \includegraphics[scale=0.21, valign=t]{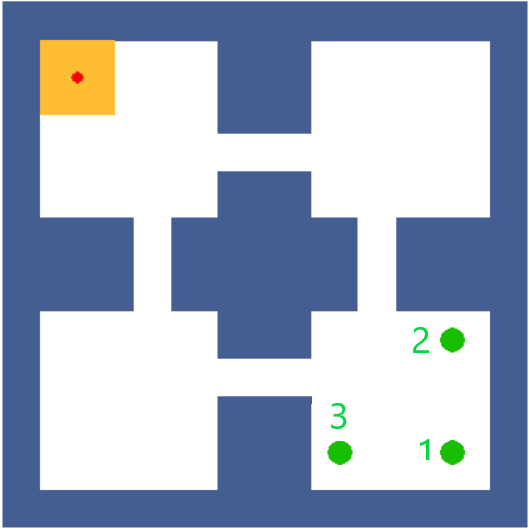}
        \hspace{0.5cm}
        \includegraphics[valign=t]{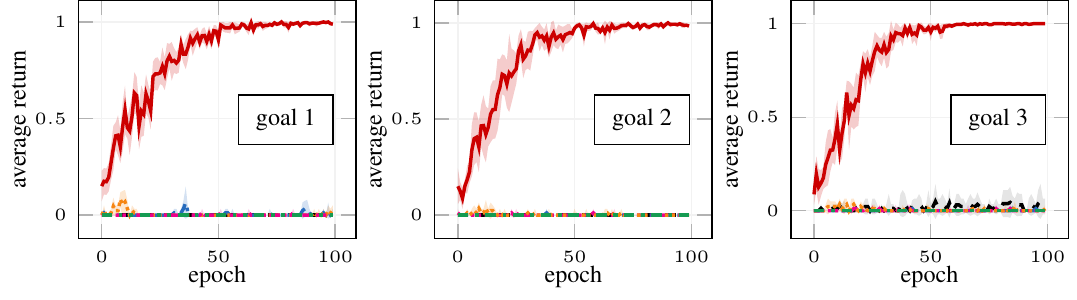}
        \vspace{-0.2cm}
		\caption{GridWorld Navigation}
		\label{fig:goal_rl_grid_world}
	\end{center}
	\end{subfigure}
	\vspace{0.2cm}

	\begin{subfigure}[b]{0.49\textwidth}
	\begin{center}
	    \includegraphics[valign=t]{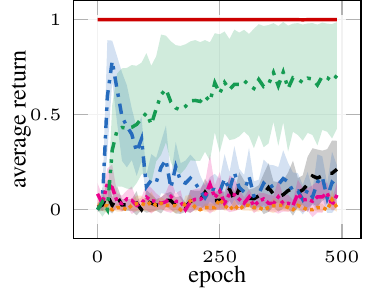}
	    \includegraphics[scale=0.56, valign=t]{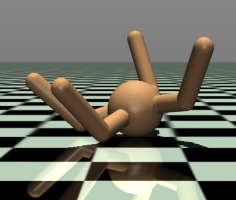}
	    \vspace{-0.3cm}
	    \caption{Ant Escape}
	    \label{fig:goal_rl_ant_escape}
	\end{center}
	\end{subfigure}
	\begin{subfigure}[b]{0.49\textwidth}
	\begin{center}
	    \includegraphics[valign=t]{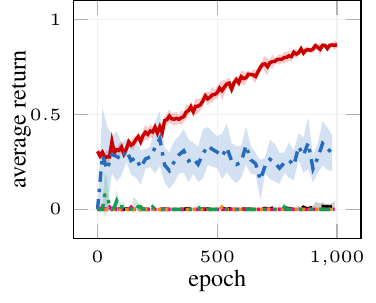}
	    \hspace{-0.2cm}
	    \includegraphics[scale=0.57, valign=t]{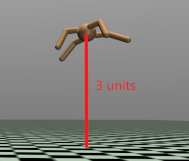}
	    \vspace{-0.3cm}
	    \caption{Ant Jump}
	    \label{fig:goal_rl_ant_jump}
	\end{center}
	\end{subfigure}
	\vspace{0.2cm}

	\begin{subfigure}[b]{0.49\textwidth}
	\begin{center}
	    \includegraphics[valign=t]{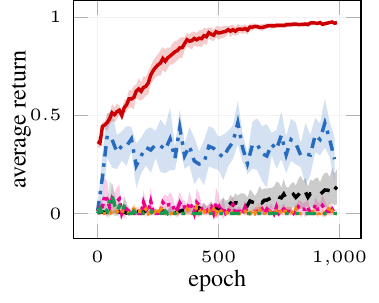}
	    \hspace{-0.2cm}
	    \includegraphics[scale=0.46, valign=t]{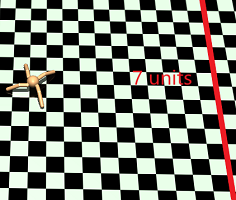}
	    \vspace{-0.3cm}
	    \caption{Ant Navigate}
	    \label{fig:goal_rl_ant_navigate}
	\end{center}
	\end{subfigure}
	\begin{subfigure}[b]{0.49\textwidth}
	\begin{center}
	    \includegraphics[valign=t]{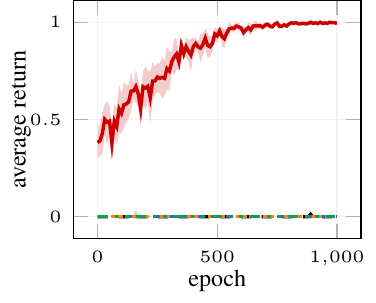}
	    \hspace{-0.2cm}
	    \includegraphics[scale=0.565, valign=t]{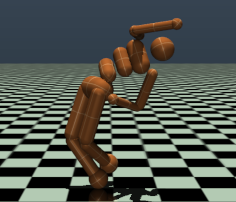}
	    \vspace{-0.3cm}
	    \caption{Humanoid Up}
		\label{fig:goal_rl_humanoid_up}
	\end{center}
	\end{subfigure}
	\vspace{0.2cm}

\centering \includegraphics[]{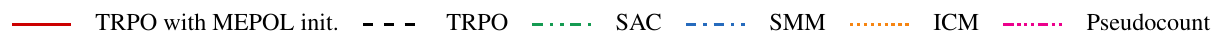}
\vspace{-0.2cm}
\caption{Comparison of the average return as a function of learning epochs achieved by TRPO with MEPOL initialization, TRPO, SAC, SMM, ICM, and Pseudocount over a set of sparse-reward RL tasks. For each task, we report a visual representation and learning curves. (95\% c.i. over 8 runs).}
\label{fig:goal_rl}
\end{figure*}
In this section, we illustrate how a learning agent can benefit from an exploration policy learned by MEPOL when dealing with a variety of goal-based RL tasks.
Especially, we compare the performance achieved by TRPO~\cite{schulman2015trpo} initialized with a MEPOL policy (the one we learned in Section~\ref{sec:exp_tae}) \wrt a set of significant baselines that learn from scratch, \ie starting from a randomly initialized policy. These baselines are: TRPO, SAC~\cite{haarnoja2018soft}, which promotes exploration over actions, SMM~\cite{lee2019efficient}, which has an intrinsic reward related to the state-space entropy, ICM~\cite{pathak2017curiosity}, which favors exploration by fostering prediction errors, and Pseudocount~\cite{bellemare2016unifying}, which assigns high rewards to rarely visited states.
The algorithms are evaluated in terms of average return on a series of sparse-reward RL tasks defined over the environments we considered in the previous sections.

Note that we purposefully chose an algorithm without a smart exploration mechanism, \ie TRPO, to employ the MEPOL initialization. In this way we can clearly show the merits of the initial policy in providing the necessary exploration. However, the MEPOL initialization can be combined with any other RL algorithm, potentially improving the reported performance.
In view of previous results in task-agnostic exploration learning (Section~\ref{sec:exp_tae}), where MaxEnt is plainly dominated by our approach, we do not compare with TRPO initialized with a MaxEnt policy, as it would not be a challenging baseline in this setting.

In GridWorld, we test three navigation tasks with different goal locations (see Figure~\ref{fig:goal_rl_grid_world}). The reward is 1 in the states having Euclidean distance to the goal lower than 0.1.
For the Ant environment, we define three, incrementally challenging, tasks: Escape, Jump, Navigate.
In the first, the Ant starts from an upside-down position and it receives a reward of 1 whenever it rotates to a straight position (Figure~\ref{fig:goal_rl_ant_escape}).
In Jump, the agent gets a reward of 1 whenever it jumps higher than three units from the ground (Figure~\ref{fig:goal_rl_ant_jump}).
In Navigate, the reward is 1 in all the states further than 7 units from the initial location (Figure~\ref{fig:goal_rl_ant_navigate}).
Finally, in Humanoid Up, the agent initially lies on the ground and it receives a reward of 1 when it is able to stand-up (Figure~\ref{fig:goal_rl_humanoid_up}).
In all the considered tasks, the reward is zero anywhere except for the goal states, an episode ends when the goal is reached.

As we show in Figure~\ref{fig:goal_rl}, the MEPOL initialization leads to a striking performance across the board, while the tasks resulted extremely hard to learn from scratch.
In some cases (Figure~\ref{fig:goal_rl_ant_escape}), MEPOL allows for zero-shot policy optimization, as the optimal behavior has been already learned in the unsupervised exploration stage. In other tasks (\eg Figure~\ref{fig:goal_rl_grid_world}), the MEPOL-initialized policy has lower return, but it permits for lighting fast adaptation \wrt random initialization.
Note that, to match the tasks' higher-level of abstraction, in Ant Navigate and Humanoid Up we employed MEPOL initialization learned by maximizing the entropy over mere spatial coordinates (x-y in Ant, x-y-z in Humanoid). However, also the exact policies learned in Section~\ref{sec:exp_tae} fares remarkably well in those scenarios (see Appendix~\ref{apx:higher_lower_level}), albeit experiencing slower convergence.

\section{Discussion and Conclusions}
\label{sec:conclusions}

In this paper, we addressed task-agnostic exploration in environments with non-existent rewards by pursuing state entropy maximization. We presented a practical policy-search algorithm, MEPOL, to learn an optimal task-agnostic exploration policy in continuous, high-dimensional domains. We empirically showed that MEPOL performs outstandingly in terms of state entropy maximization, and that the learned policy paves the way for solving several reward-based tasks downstream.

\paragraph{Extensions and Future Directions} 
First, we note that the results reported for the goal-based setting (Section~\ref{sec:exp_goal}) can be easily extended, either considering a wider range of tasks or combining the MEPOL initialization with a variety of RL algorithms (other than TRPO). In principle, any algorithm can benefit from task-agnostic exploration, especially when dealing with sparse-reward tasks. 
Secondly,  while we solely focused on finite-horizon exploration, it is straightforward to adapt the presented approach to the discounted case: We could simply generate a batch of trajectories with a probability $1 - \gamma$ to end at any step instead of stopping at step $T$, and then keep everything else as in Algorithm~\ref{alg:MEPOL}. This could be beneficial when dealing with discounted tasks downstream. 
Future work might address an adaptive control over the exploration horizon $T$, so to induce a curriculum of exploration problems, starting from an easier one (with a short $T$) and going forward to more challenging problems (longer $T$).
Promising future directions also include learning task-agnostic exploration across a collection of environments, and contemplating the use of a non-parametric state entropy regularization in reward-based policy optimization.

\paragraph{Other Remarks}
It is worth mentioning that the choice of a proper metric for the $k$-NN computation might significantly impact the final performance. In our continuous control experiments, we were able to get outstanding results with a simple Euclidean metric. However, different domains, such as learning from images, might require the definition of a more thoughtful metric space in order to get reliable entropy estimates. In this regard, some recent works~\cite[\eg][]{misra2020kinematic} provide a blueprint to learn state embeddings in reward-free rich-observation problems.
Another theme that is worth exploring to get even better performance over future tasks is sample reuse. In MEPOL, the samples collected during task-agnostic training are discarded, while only the resulting policy is retained. An orthogonal line of research focuses on the problem of collecting a meaningful batch of samples in a reward-free setting~\cite{jin2020reward}, while discarding sampling policies. Surely a combination of the two objectives will be necessary to develop truly efficient methods for task-agnostic exploration, but we believe that these two lines of work still require significant individual advances before being combined into a unique, broadly-applicable approach.

To conclude, we hope that this work can shed some light on the great potential of state entropy maximization approaches to perform task-agnostic exploration.

\bibliography{biblio}
\bibliographystyle{aaai.bst}

\clearpage
\onecolumn
\appendix

\section{Related Work}
\label{apx:related_work}

Our work falls into the category of unsupervised reinforcement learning. Especially, we address the problem of task-agnostic exploration, \ie how to learn an unsupervised exploration policy that generalizes towards a wide range of tasks. This line of work relates to the recent reward-free exploration framework, to intrinsically-motivated learning, state-reaching approaches, and, to some extent, meta-RL. In this section, we provide a non-exhaustive review of previous work in these areas, and we further discuss those that relates the most with ours.

The reward-free exploration framework~\cite{jin2020reward}, which has received notable attention lately~\cite[\eg][]{tarbouriech2020active, zhang2020task, kaufmann2020adaptive}, consider a very similar setting \wrt the one we address in this work, but a mostly orthogonal objective. While they strive to collect a meaningful batch of samples in the reward-free phase, possibly disregarding the sampling strategy, our focus is precisely on getting an effective sampling policy out of the task-agnostic phase, and we do not make use of the training samples.

In the context of (supervised) reinforcement learning, many works have drawn inspiration from intrinsic motivation~\cite{chentanez2005intrinsically, oudeyer2007intrinsic} to design \emph{exploration bonuses} that help the agent overcoming the exploration barrier in sparsely defined tasks.
Some of them, initially~\cite{schmidhuber1991possibility} and more recently~\cite{pathak2017curiosity, burda2018large}, promotes exploration by fostering prediction errors of an environment model, which is concurrently improved during the learning process.
Another approach~\cite[\eg][]{lopes2012exploration, houthooft2016vime} considers exploration strategies that maximizes the information gain, rather than prediction errors, on the agent's belief about the environment.
Other works devise intrinsic rewards that are somewhat proportional to the novelty of a state, so that the agent is constantly pushed towards new portions of the state space. Notable instances are count-based methods~\cite[\eg][]{bellemare2016unifying, tang2017exploration}, and random distillation~\cite{burda2018exploration}.
Lately, \citet{lee2019efficient} propose an exploration bonus that is explicitly related to the state distribution entropy, while they present an interesting unifying perspective of other intrinsic bonuses in the view of entropy maximization.

In the intrinsic motivation literature, other works tackle the problem of learning a set of useful skills in an unsupervised setting~\cite{gregor2017variational, achiam2018variational, eysenbach2018diversity}. The common thread is to combine variational inference and intrinsic motivation to maximize some information theoretic measure, which is usually closely related to \emph{empowerment}~\cite{salge2014empowerment, mohamed2015variational}, \ie the ability of the agent to control its own environment. In~\cite{gregor2017variational, achiam2018variational}, the proposed methods learn a set of diverse skills by maximizing the mutual information between the skills and their terminations states. \citet{eysenbach2018diversity} consider a mutual information objective as well, but computed over all the states visited by the learned skills, instead of just termination states.

Another relevant line of research focus on \emph{state-reaching objectives}. In tabular settings, \citet{lim2012autonomous} cast the problem in its most comprehensive form, proposing a principled method to incrementally solve any possible state-reaching task, from the easiest to the hardest. In~\cite{gajane2019autonomous}, the same method is extended to non-stationary environments. Other works, originally~\cite{bonarini2006incremental, bonarini2006self} and lately~\cite{ecoffet2019go}, provide algorithms to address a relaxed version of the state-reaching problem. In their case, the goal is to learn reaching policies to return to some sort of promising states (instead of all the states), then to seek for novel states from there on. In complex domains, \citet{pong2019skew} consider learning maximum-entropy goal generation to provide targets to the state-reaching component. Notably, they relates this procedure to maximum-entropy exploration.

\emph{Meta-learning}~\cite{schmidhuber1987meta} has been successfully applied to address generalization in RL~\citep[\eg][]{finn2017maml}. In this set-up, the agent faces numerous tasks in order to meta-learn a general model, which, then, can be quickly adapted to solve unseen tasks in a few shots. These methods generally relies on reward functions to train the meta-model. Instead, a recent work~\cite{gupta2018unsupervised} considers an unsupervised meta-RL set-up, which has some connections with our work. However, their focus is directed on learning fast adaptation rather than exploration, which is delegated to a method~\cite{eysenbach2018diversity} we already mentioned.

\paragraph{Maximum-Entropy Exploration} As reported in Section~\ref{sec:introduction}, the work that is more relevant to the context of this paper is the one on maximum-entropy exploration~\cite{hazan19maxent, tarbouriech2019active, mutti2019ideal, lee2019efficient}. While \cite{tarbouriech2019active, mutti2019ideal} only focus on asymptotic distributions and tabular settings, the work in~\cite{hazan19maxent, lee2019efficient} consider finite-horizon distributions and continuous domains as we do. Thus, it is worth reiterating the main differences between the proposed algorithm (MEPOL) and previous approaches (MaxEnt~\cite{hazan19maxent} and SMM~\cite{lee2019efficient}).
First, MEPOL learns a single exploration policy maximizing the entropy objective, MaxEnt and SMM learn a mixture of policies that collectively maximizes the entropy.
Secondly, MaxEnt and SMM relies on state density modeling to estimate the entropy,\footnote{Note that entropy estimation on top of density estimation has well-known shortcomings~\cite{beirlant1997nonparametric}.} MEPOL does not have to learn any explicit model as it performs non-parametric entropy estimation.
Lastly, MEPOL does not require intrinsic rewards, in contrast to MaxEnt and SMM that optimize intrinsic reward functions.\footnote{These are pseudo-reward functions actually~\cite{lee2019efficient}, as they depend on the current policy.}
As a side note, the native objective of SMM is to match a given target state distribution, which reduces to entropy maximization when the target is uniform. A relevant state-distribution matching approach, albeit recasted in an imitation learning set-up, has been also developed in~\cite{ghasemipour2019divergence}.

\section{Proofs}
\label{apx:proof}

\estimatorBias*

\begin{proof}
    The proof follow the sketch reported in~\citep[Section 4.1]{ajgl2011differential}. 
    First, We consider the estimator $\widehat{G}_k (f' | f) = \widehat{H}_k (f' | f) - \ln k + \Psi (k)$, that is:  
    \begin{equation}
        \widehat{G}_k (f' | f) = \sum_{i = 1}^N \frac{W_i}{k} \ln \frac{V_i^k}{W_i}.
        \label{eq:biased_estimator}
    \end{equation}
    By considering its expectation \wrt the sampling distribution we get:
    \begin{equation*}
        \EV_{x \sim f} [ \widehat{G}_k (f'|f) ] =
        \EV_{x \sim f} \bigg[ \sum_{i = 1}^N \frac{W_i}{k} \ln 
        \frac{1}{W_i}
        \frac{ R_i^p \uppi^{\nicefrac{p}{2}}}
        { \Gamma (\frac{p}{2} + 1)} \bigg],
    \end{equation*}
    where, for the sake of clarity, we will replace each logarithmic term as:
    \begin{equation}
        T_i = \ln \frac{1}{\sum_{j \in \mathcal{N}_i^k} w_j} \frac{ R_i^p \uppi^{\nicefrac{p}{2}}}{ \Gamma (\frac{p}{2} + 1)}.
        \label{eq:logarithmic_terms}
    \end{equation}
    Since we are interested in the asymptotic mean, we can notice that for $N \to \infty$ we have $\nicefrac{W_i}{k} \to \nicefrac{\overline{w}_i}{N}$, where $\overline{w}_i = \nicefrac{f'(x)}{f(x)}$ are the unnormalized importance weights~\citep[Section 4.1]{ajgl2011differential}. Thus, for $N \to \infty$, we can see that:
    \begin{equation*}
        \EV_{x \sim f} \big[\widehat{G}_k (f'|f) \big] =
        \EV_{x \sim f} \bigg[ \sum_{i = 1}^N \frac{\overline{w}_i}{N} T_i \bigg] =  
        \EV_{x \sim f'} \bigg[\frac{1}{N} \sum_{i = 1}^N T_i \bigg],
    \end{equation*}
    where the random variables $T_1, T_2, \ldots, T_N$ are identically distributed, so that:
    \begin{equation*}
        \EV_{x \sim f} \big[\widehat{G}_k (f'|f) \big] = 
        \EV_{x \sim f'} \big[ T_1 \big].
    \end{equation*}
    Thus, to compute the expectation, we have to characterize the following probability for any real number $r$ and any $x \sim f'$:
    \begin{equation*}
        Pr \big[ T_1 > r | X_1 = x \big] = Pr \big[ R_1 > \rho_r | X_1 = x \big],
    \end{equation*}
    where we have:
    \begin{equation*}
        \rho_r = \bigg[ \frac{W_i \cdot  \Gamma(\frac{p}{2} + 1) \cdot e^r}
        {\uppi^{\nicefrac{p}{2}} } \bigg]^{\frac{1}{p}}.
    \end{equation*}
    We can rewrite the probability as a binomial:
    \begin{equation*}
        Pr \big[ R_1 > \rho_r | X_1 = x \big] =
        \sum_{i = 0}^{k - 1} \binom{N - 1}{i} \big[ Pr( S_{\rho_r, x} ) \big ]^i \big[ 1 - Pr( S_{\rho_r, x} ) \big]^{N - 1 - i},
    \end{equation*}
    where $P( S_{\rho_r, x} )$ is the probability of $x$ lying into the sphere of radius $\rho_r$, denoted as $S_{\rho_r, x}$. Then, we employ the Poisson Approximation~\citep{hodges1960poisson} to this binomial distribution, reducing it to a poisson distribution having parameter:
    \begin{align*}
        \lim_{N \to \infty} \big[ N Pr(S_{\rho_r, x}) \big]
        =& \lim_{N \to \infty} \bigg[ N f(x) \frac{\uppi^{\nicefrac{p}{2}} \cdot \rho_r^p}{ \Gamma(\frac{p}{2} + 1) } \bigg] \\
        =& \lim_{N \to \infty} \big[ N W_i f(x) e^r \big]
        = \overline{w}_i f(x) k e^r = f'(x) k e^r.
    \end{align*}
    Therefore, we get:
    \begin{equation*}
        \lim_{N \to \infty} Pr \big[ T_1 > r | X_1 = x \big]
        = \sum_{i = 0}^{k - 1} \frac{ \big[ k f'(x) e^r\big]^i }{ i! } e^{- k f'(x) e^r} = Pr \big[ T_x > r \big],
    \end{equation*}
    such that the random variable $T_x$ has the pdf:
    \begin{equation*}
        h_{T_x} (y) = \frac{ \big[ k f'(x) e^r\big]^k }{ (k - 1)! } e^{- k f'(x) e^r}, -\infty < y < \infty.
    \end{equation*}
    Finally, we can compute the expectation following the same steps reported in~\citep[Theorem 8]{singh2003nearest}:
    \begin{align*}
        \lim_{N \to \infty} \EV \big[ T_1 | X_1 = x \big] &= \int_{-\infty}^{\infty} y \frac{ \big[ k f'(x) e^r\big]^k }{ (k - 1)! } e^{- k f'(x) e^r} \de y \\
        &= \int_{0}^{\infty} \big[ \ln z - \ln k - \ln f'(x) \big] \frac{ \big[ z^{k - 1} \big] }{ (k - 1)! } e^{- z}
        \de z \\
        &= \frac{1}{\Gamma (k) } \int_{0}^{\infty} \big[ \ln(z) z^{k - 1} e^{-z} \big] \de z - \ln k - \ln f'(x) \\
        &= \Psi (k) - \ln k - \ln f'(x),
    \end{align*}
    which for a generic $x$ it yields:
    \begin{equation*}
        \lim_{N \to \infty} \EV_{x \sim f'} \big[  T_1 \big] = H(f') - \ln k + \Psi (k) = \lim_{N \to \infty} \EV_{x \sim f} \big[ \widehat{G}_k (f' | f) \big].
    \end{equation*}
\end{proof}

\estimatorVariance*

\begin{proof}
    We consider the limit of the variance of $\widehat{H}_k (f'| f)$, we have:
    \begin{align*}
        \lim_{N \to \infty} \Var_{x \sim f} \big[ \widehat{H}_k (f'| f) \big] 
        &= \lim_{N \to \infty} \Var_{x \sim f} \big[ \widehat{G}_k (f'| f) \big] 
        = \lim_{N \to \infty} \Var_{x \sim f} \bigg[ \sum_{i = 1}^N \frac{W_i}{k} T_i \bigg]
        = \lim_{N \to \infty} \Var_{x \sim f} \bigg[  \frac{1}{N} \sum_{i = 1}^N \overline{w}_i T_i \bigg],
    \end{align*}
    where $\widehat{G}_k (f'| f)$ is the estimator without the bias correcting term~\eqref{eq:biased_estimator}, and $T_i$ are the logarithmic terms~\eqref{eq:logarithmic_terms}.
    Then, since the distribution of the random vector $\big(\overline{w}_1 T_1, \overline{w}_2 T_2, \ldots, \overline{w}_N T_N \big)$ is the same as any permutation of it~\citep{singh2003nearest}:
    \begin{equation*}
        \Var_{x \sim f} \bigg[  \frac{1}{N} \sum_{i = 1}^N \overline{w}_i T_i \bigg] = \frac{\Var_{x \sim f} \big[ \overline{w}_1 T_1\big]}{N} + \frac{N (N - 1)}{N^2} \Cov \big( \overline{w}_1 T_1, \overline{w}_2 T_2 \big).
    \end{equation*}
    Assuming that, for $N \to \infty$, the term $\Cov ( \overline{w}_1 T_1, \overline{w}_2 T_2 ) \to 0$ as its non-IW counterpart~\citep[Theorem 11]{singh2003nearest}, we are interested on the first term $\Var_{x \sim f} \big[ \overline{w}_1 T_1\big]$. Especially, we can derive:
    \begin{align*}
        \Var_{x \sim f} \big[ \overline{w}_1 T_1 \big]
        &= \Var_{x \sim f} \bigg[ \overline{w}_1 \ln 
        \frac{1}{\overline{w}_1} \frac{N}{k} \frac{R_1^p \uppi^{\nicefrac{p}{2}}}{ \Gamma (\frac{p}{2} + 1)} \bigg] = \Var_{x \sim f} \bigg[ - \overline{w}_1 \ln \overline{w}_1 +  \overline{w}_1 \ln R_1^p + \overline{w}_1 \ln \frac{N}{k} 
        \frac{ \uppi^{\nicefrac{p}{2}} }{ \Gamma (\frac{p}{2} + 1)} \bigg],
    \end{align*}
    where in the following, we will substitute $C = \frac{N \uppi^{\nicefrac{p}{2}}}{k \Gamma (\nicefrac{p}{2} + 1)}$.
    Then, we can write the second momentum as:
    \begin{align*}
        \EV_{x \sim f} \bigg[ \big( \overline{w}_1 T_1 \big)^2 \bigg]
        &= \EV_{x \sim f} \bigg[ 
        \big( \overline{w}_1 \ln \overline{w}_1 \big)^2 
        + \big( \overline{w}_1 \ln R_1^p \big)^2 
        + \big( \overline{w}_1 \ln C \big)^2 \\
        &\qquad\qquad 
        - 2 \overline{w}_1^2 \ln \overline{w}_1 \ln R_1^p 
        - 2 \overline{w}_1^2 \ln \overline{w}_1 \ln C
        + 2 \overline{w}_1^2 \ln R_1^p \ln C
        \bigg],
    \end{align*}
    while the squared expected value is:
    \begin{align*}
        \bigg( \EV_{x \sim f} \big[ \overline{w}_1 T_1 \big] \bigg)^2
        &= \bigg( 
        - \EV_{x \sim f} \big[ \overline{w}_1 \ln \overline{w}_1 \big]
        + \EV_{x \sim f} \big[ \overline{w}_1 \ln R_1^p \big]
        + \EV_{x \sim f} \big[ \overline{w}_1 \ln C \big] 
        \bigg)^2 \\
        &= 
        \bigg( \EV_{x \sim f} \big[ \overline{w}_1 \ln \overline{w}_1 \big] \bigg)^2
        + \bigg( \EV_{x \sim f} \big[ \overline{w}_1 R_1^p \big] \bigg)^2 + \bigg( \EV_{x \sim f} \big[ \overline{w}_1 \ln C \big] \bigg)^2 \\
        &\qquad
        - 2 \EV_{x \sim f} \bigg[ \overline{w}_1^2 \ln \overline{w}_1 \ln R_1^p \bigg] 
        - 2 \EV_{x \sim f} \bigg[ \overline{w}_1^2 \ln \overline{w}_1 \ln C \bigg]
        + 2 \EV_{x \sim f} \bigg[ \overline{w}_1^2 \ln C \ln R_1^p \bigg].
    \end{align*}
    Thus, we have:
    \begin{align*}
        \Var_{x \sim f} \big[ \overline{w}_1 T_1 \big]
        &= 
        \EV_{x \sim f} \bigg[ \big( \overline{w}_1 T_1 \big)^2 \bigg]
        - \bigg( \EV_{x \sim f} \big[ \overline{w}_1 T_1 \big] \bigg)^2 \\
        &= \Var_{x \sim f} \bigg[ \overline{w}_1 \ln \overline{w}_1 \bigg]
        + \Var_{x \sim f} \bigg[ \overline{w}_1 \ln R_1^p  \bigg]
        + \Var_{x \sim f} \bigg[ \overline{w}_1 \ln C \bigg] \\
        &= \Var_{x \sim f} \bigg[ \overline{w}_1 \ln \overline{w}_1 \bigg]
        + \Var_{x \sim f} \bigg[ \overline{w}_1 \ln R_1^p  \bigg]
        + \big(\ln C \big)^2 \Var_{x \sim f} \bigg[ \overline{w}_1 \bigg].
    \end{align*}
    Summing it up, we can write the asymptotic order of the variance as:
    \begin{align*}
        &\lim_{N \to \infty} \Var_{x \sim f} \big[ \widehat{H} (f'| f) \big] = 
        \frac{ \Var_{x \sim f} \big[ \overline{w} (x) \ln \overline{w} (x) \big]
        + \Var_{x \sim f} \big[ \overline{w} (x) \ln R (x)^p  \big]
        + \big( \ln C \big)^2 \Var_{x \sim f} \big[ \overline{w} (x) \big] }{N}
    \end{align*}
    
\end{proof}

\estimatorGradient*

\begin{proof}

We consider the IW entropy estimator~\eqref{eq:offs_estimator}, that is
\begin{equation}
    \widehat{H}_k ( \bar{d}_T (\vtheta') | \bar{d}_T (\vtheta) ) = - \sum_{i = 1}^N \frac{W_i}{k} \ln \frac{W_i}{V_i^k} + \ln k - \Psi (k),
    \label{eq:iw_estimator_p}
\end{equation}
where:
\begin{equation*}
    W_i = \sum_{j \in \mathcal{N}_i^k} w_j 
    = \sum_{j \in \mathcal{N}_i^k} 
    \frac{ \prod_{z=0}^t \frac{ \pi_{\vtheta'} ( a_j^z | s_j^z ) }{ \pi_{\vtheta} ( a_j^z | s_j^z ) } }
    {\sum_{n = 1}^N \prod_{z=0}^t \frac{ \pi_{\vtheta'} ( a_n^z | s_n^z ) }{ \pi_{\vtheta} ( a_n^z | s_n^z ) } }.
\end{equation*}
Then, by differentiating Equation~\eqref{eq:iw_estimator_p} \wrt $\vtheta'$, we have:
\begin{align}
    \nabla_{\vtheta'} \widehat{H}_k ( \bar{d}_T (\vtheta') | \bar{d}_T (\vtheta) )
    &= - \sum_{i = 1}^{N} \nabla_{\vtheta'} \bigg( \frac{ \sum_{j \in \mathcal{N}_i^k} w_j }{ k } \ln  \frac{ \sum_{j \in \mathcal{N}_i^k} w_j }{ V_i^k } + \ln k - \psi (k) \bigg) \nonumber \\
    &= - \sum_{i = 1}^{N} \bigg( \frac{ \sum_{j \in \mathcal{N}_i^k}  \nabla_{\vtheta'} w_j }{ k } \ln  \frac{ \sum_{j \in \mathcal{N}_i^k} w_j }{ V_i^k } + \frac{ \sum_{j \in \mathcal{N}_i^k} w_j }{ k } \frac{ V_i^k }{ \sum_{j \in \mathcal{N}_i^k} w_j } \sum_{j \in \mathcal{N}_i^k}  \nabla_{\vtheta'} w_j \bigg) \nonumber \\
    &=  - \sum_{i = 1}^{N} \frac{ \sum_{j \in \mathcal{N}_i^k}  \nabla_{\vtheta'} w_j }{ k } \bigg( V_i^k + \ln \frac{ \sum_{j \in \mathcal{N}_i^k} w_j }{ V_i^k } \bigg). \label{eq:iw_estimator_gradient_p}
\end{align}
Finally, we consider the expression of $\nabla_{\vtheta'} w_j$ in Equation~\eqref{eq:iw_estimator_gradient_p} to conclude the proof:
\begin{align*}
    \nabla_{\theta} w_j 
    &= w_j \nabla_{\vtheta'} \ln w_j \\
    &= w_j \nabla_{\vtheta'} \bigg( 
    \ln \prod_{z=0}^t \frac{ \pi_{\vtheta'} ( a_j^z | s_j^z ) }{ \pi_{\vtheta} ( a_j^z | s_j^z ) } 
    - \ln \sum_{n = 1}^{N}  \overbrace{\prod_{z=0}^t \frac{ \pi_{\vtheta'} ( a_n^z | s_n^z ) }{ \pi_{\vtheta} ( a_n^z | s_n^z ) } }^{\bm{Prod}_n} \bigg) \\
    &= w_j \bigg( \sum_{z = 0}^{t} \nabla_{\vtheta'} \ln \pi_{\vtheta'} ( a_j^z | s_j^z ) 
    - \frac{ \sum_{n = 1}^{N} \nabla_{\vtheta'} \bm{Prod}_n }
    { \sum_{n = 1}^{N}  \bm{Prod}_n} \bigg) \\
    &= w_j \bigg( \sum_{z = 0}^{t} \nabla_{\vtheta'} \ln \pi_{\vtheta'} ( a_j^z | s_j^z ) 
    - \frac{ \sum_{n = 1}^N \bm{Prod}_n \nabla_{\vtheta'} \ln \big( \bm{Prod}_n \big) }
    { \sum_{n = 1}^N \; \bm{Prod}_n } \bigg) \\
    &= w_j \bigg( \sum_{z = 0}^{t} \nabla_{\vtheta'} \ln \pi_{\vtheta'} ( a_j^z | s_j^z ) 
    - \frac{ \sum_{n = 1}^N  \big( \bm{Prod}_n \sum_{z = 0}^{t} \nabla_{\vtheta'} \ln \pi_{\vtheta'} (a_n^z | s_n^z) \big)}
    { \sum_{n = 1}^N \bm{Prod}_n } \bigg).
\end{align*}

\end{proof}

\section{Empirical Analysis: Further Details}
\label{apx:experiments}

\subsection{Environments}
For all the environments, we use off-the-shelf implementations from the OpenAI gym library \citep{brockman2016openai} with the exception of GridWorld, which we coded from scratch and we describe in the next paragraph.
We also slightly modified the MountainCar environment (see Figure~\ref{fig:exp_setup_env}) by adding a wall on top of the right mountain to make the environment non-episodic.

In GridWorld, the agent can navigate a map composed of four rooms connected by four hallways, as represented in Figure~\ref{fig:exp_setup_env}.
At each step the agent can choose how much to move on the x and y axes. The maximum continuous absolute change in position along any of the axes is 0.2. Each room is a space of 5 by 5 units, thus, the agent needs around 50 steps to move, along a straight line, from one side of the room to the other.
Any action that leads the agent to collide with a wall is ignored and the agent remains in the previous state (position).

\begin{figure}[h!]
\begin{center}
\includegraphics[scale=0.69]{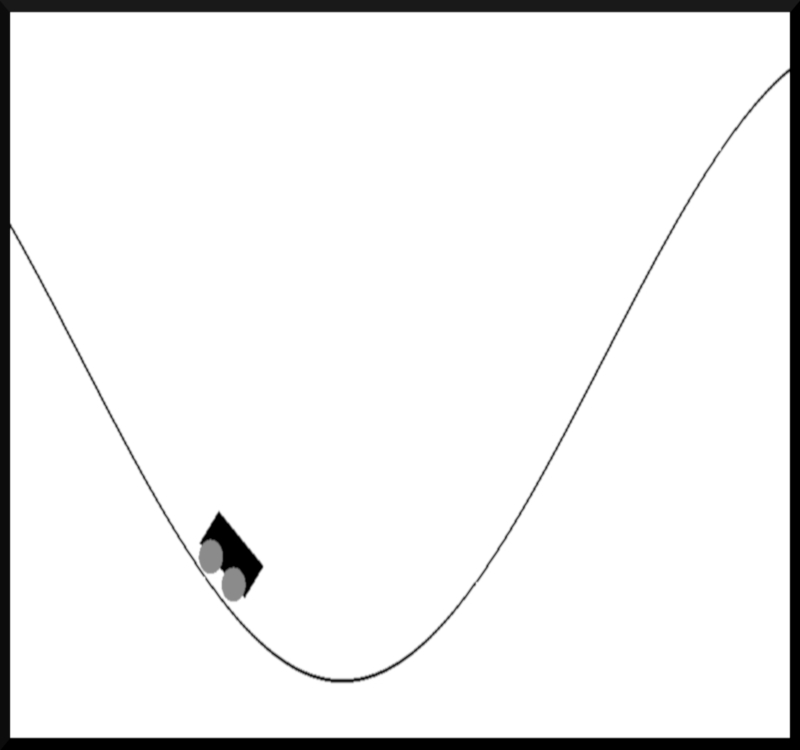}
\hspace{2cm}
\includegraphics[scale=0.4]{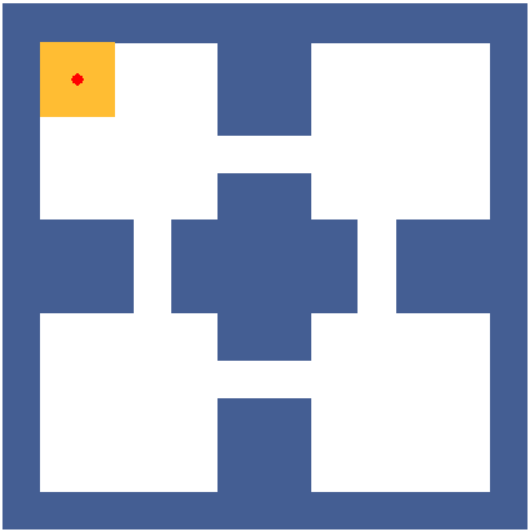}
\vskip 0.2in
\caption{Visual representation of the MountainCar (left) and GridWorld (right) environments. In Gridworld, the agent is represented with the red circle, and it starts each episode in a random position inside the yellow area.}
\label{fig:exp_setup_env}
\end{center}
\end{figure}

\subsection{Class of Policies}
\label{apx:exp_setup_policy}
In all the experiments, the policy is a Gaussian distribution with diagonal covariance matrix. It takes as input the environment state features and outputs an action vector $a \sim \mathcal{N}(\mu,\,\sigma^{2})$. The mean $\mu$ is state-dependent and is the downstream output of a densely connected neural network. The standard deviation is state-independent and it is represented by a separated trainable vector. The dimension of $\mu$, $\sigma$ and $a$ vectors is equal to the action-space dimension of the environment.

\subsection{Task-Agnostic Exploration Learning}

\subsubsection{Continuous Control Set-Up}
\label{apx:tae_setup}

Here we further comment the experimental set-up for continuous control domains that we have briefly described in Section~\ref{sec:exp_tae}, especially concerning how we select the set of features on which the entropy index is maximized.
First, in the Ant domain, we maximize the entropy over a 7D space of spatial coordinates (3D) and torso orientation (4D), excluding joint angles and velocities. This is to obtain an intermediate hurdle, in terms of dimensionality, \wrt the smaller GridWorld and MountainCar, and the most complex Humanoid and HandReach.
Instead, in the latter two domains, we essentially maximize the entropy over the full state space excluding velocities and external forces, so that we have a 24D space both in Humanoid (3D position, 4D body orientation, and all the joint angles) and HandReach (all the joint angles). We noted that including external forces does not affect the entropy maximization in a meaningful way, as they resulted always zero during training. The reason why we also discarded velocities from the entropy computation is twofold. First, we noted that velocity-related features are quite unstable and spiky, so that they can be harmful without normalization, which we avoided. Secondly, we think that maximizing the entropy over the largest set of features is not necessarily a good idea when targeting generalization (see~\ref{apx:higher_lower_level}), especially if improving the entropy over some features reduces the entropy over some other (as in the case of positions/angles and velocities).

\subsubsection{MEPOL}

You can find the implementation of MEPOL at \url{https://github.com/muttimirco/mepol}. As outlined in the pseudcode of Algorithm \ref{alg:MEPOL}, in each epoch a dataset of particles $\mathcal{D}_\tau = \{ (\tau_i^t, s_i) \}_{i = 1}^N$ is gathered for the given time horizon $T$. We call $N_{traj}$ the number of trajectories, or batch size, used to build the dataset, so that $N = N_{traj} * T$. Before starting the main loop of the algorithm we perform some training steps to force the policy to output a zero $\mu$ vector. This is instrumental to obtain a common starting point across all the seeds and can be removed without affecting the algorithm behavior.

For what concerns the k-nearest neighbors computation we use the \textit{neighbors} package from the \textit{scikit-learn} library~\citep{scikit-learn}, which provides efficient algorithms and data structures to first compute and then query nearest neighbors (KD-tree in our case). Note that the computational complexity of each epoch is due to the cost of computing the $k$-NN entropy estimation, which is $\mathcal{O} (p N \log N  )$ to build the tree ($p$ is the number of dimensions), and $\mathcal{O} (k N \log N)$ to search the neighbors for every sample. Summing it up we get a complexity in the order of $\mathcal{O} (N \log N (p + k))$ for each epoch.

\label{sec:mepol_param}
\begin{table}[H]
    \centering
    \caption{MEPOL Parameters}
    \begin{tabular}[t]{lccccc}
        \hline
        &MountainCar &GridWorld &Ant &Humanoid &HandReach\\
        \hline
        Number of epochs &650 &200 &2000 &2000 &2000 \\
        Horizon (T) &400 &1200 &500 &500 &50\\
        Batch Size ($N_{traj}$) &20 &20 &20 &20 &50 \\
        Kl threshold ($\delta$) &15.0 &15.0 &15.0 &15.0 &15.0\\
        Learning rate ($\alpha$) &$10^{-4}$ &$10^{-5}$ &$10^{-5}$ &$10^{-5}$ &$10^{-5}$ \\
        Max iters &30 &30 &30 &30 &30 \\
        Number of neighbors (k) &4 &50 &4 &4 &4\\
        Policy hidden layer sizes &(300,300) &(300,300) &(400,300) &(400,300) &(400,300)\\
        Policy hidden layer act. function &ReLU &ReLU &ReLU &ReLU &ReLU\\
        Number of seeds &8 &8 &8 &8 &8\\
        \hline
    \end{tabular}
\end{table}%

\subsubsection{MaxEnt}
\label{apx:max_ent}

As already outlined in Section \ref{sec:exp_tae}, we use the original MaxEnt implementation to deal with continuous domains
(\url{https://github.com/abbyvansoest/maxent/tree/master/humanoid}). We adopt this implementation also for the Ant and MountainCar environments, which were originally presented only as discretized domains (\url{https://github.com/abbyvansoest/maxent_base} and \url{https://github.com/abbyvansoest/maxent/tree/master/ant}). This allowed us not only to work in a truly continuous setting, as we do in MEPOL, but also to obtain better results than the discretized version. The only significant change we made is employing TRPO instead of SAC for the RL component. This is because, having tested MaxEnt with both the configurations, we were able to get slightly superior performance, and a more stable behavior, with TRPO.

In the tables below, you can find all the parameters and corresponding ranges (or sets) of values over which we searched for their optimal values. The TRPO parameters are reported employing the same notation as in~\citep{duan2016benchmarking}.
To estimate the density of the mixture, in Ant and Humanoid, we use a higher time horizon ($T_d$) than the optimized one ($T$). The reason why we do this is that, otherwise, we were not able to obtain reliable density estimations. The batch size used for the density estimation is denoted as $N_{traj\_d}$.
We also report the number of neighbors ($k$), which does not affect the learning process of MaxEnt, that we used to calculate the entropy of the mixture in the plots of Section~\ref{sec:exp_tae}. The entropy is computed over the same time horizon $T$ and the same batch size $N_{traj}$ used in MEPOL. Note that the horizon reported for TRPO is the same as the objective horizon $T$. The neural policy architectures are not reported as they are the same as in Section~\ref{sec:mepol_param}.

\begin{table}[H]
    \centering
    \caption{MaxEnt Parameters - MountainCar}
    \begin{tabular}[t]{lcc}
        \hline
        &Value &Search In\\
        \hline
        Number of neighbors (k) &4 &-\\
        Mixture Size &60 &-\\
        Density Horizon ($T_d$) &400 &-\\
        Density Batch Size ($N_{traj\_d}$) &100 &-\\
        KDE Kernel &Epanechnikov &Gaussian, Epanechnikov\\
        KDE Bandwidth &0.1 &[0.1, 2.0]\\
        PCA &No &Yes, No\\
        \hline
        & \multicolumn{2}{c}{TRPO}\\
        \hline
        Num. Iter. &50 &{40, 50, 60}\\
        Horizon &400 &-\\
        Sim. steps per Iter. &4000 &-\\
        $\delta_{KL}$ &0.1 &0.1, 0.01, 0.001\\
        Discount ($\gamma$) &0.99 &-\\
        Number of seeds &8 &-\\
        \hline
    \end{tabular}
\end{table}%

\begin{table}[H]
    \centering
    \caption{MaxEnt Parameters - Ant}
    \begin{tabular}[t]{lcc}
       \hline
        &Value &Search In\\
        \hline
        Number of neighbors (k) &4 &-\\
        Mixture Size &30 &-\\
        Density Horizon ($T_d$) &10000 &-\\
        Density Batch Size ($N_{traj\_d}$) &10 &-\\
        KDE Kernel &Epanechnikov &Gaussian, Epanechnikov\\
        KDE Bandwidth &0.2 &[0.1, 2.0]\\
        PCA &Yes (3 components) &Yes, No\\
        \hline
        & \multicolumn{2}{c}{TRPO}\\
        \hline
        Num. Iter. &300 &{300, 500}\\
        Horizon &500 &-\\
        Sim. steps per Iter. &5000 &-\\
        $\delta_{KL}$ &0.1 &0.1, 0.008\\
        Discount ($\gamma$) &0.99 &-\\
        Number of seeds &8 &-\\
        \hline
    \end{tabular}
\end{table}%

\begin{table}[H]
    \centering
    \caption{MaxEnt Parameters - Humanoid}
    \begin{tabular}[t]{lcc}
        \hline
        &Value &Search In\\
        \hline
        Number of neighbors (k) &4 &-\\
        Mixture Size &30 &-\\
        Density Horizon ($T_d$) &50000 &-\\
        Density Batch Size ($N_{traj\_d}$) &20 &-\\
        KDE Kernel &Epanechnikov &Gaussian, Epanechnikov\\
        KDE Bandwidth &1.0 &[0.1, 2.0]\\
        PCA &No &Yes, No\\
        \hline
        & \multicolumn{2}{c}{TRPO}\\
        \hline
        Num. Iter. &300 &{200, 300}\\
        Horizon &500 &-\\
        Sim. steps per Iter. &5000 &-\\
        $\delta_{KL}$ &0.1 &0.1, 0.008\\
        Discount ($\gamma$) &0.99 &-\\
        Number of seeds &8 &-\\
        \hline
    \end{tabular}
\end{table}%

\begin{table}[H]
     \centering
     \caption{MaxEnt Parameters - HandReach}
     \begin{tabular}[t]{lcc}
         \hline
         &Value &Search In\\
         \hline
         Number of neighbors (k) &4 &-\\
         Mixture Size &30 &-\\
         Density Horizon ($T_d$) &10000 &-\\
         Density Batch Size ($N_{traj\_d}$) &20 &-\\
         KDE Kernel &Epanechnikov &Gaussian, Epanechnikov\\
         KDE Bandwidth &1.0 &[0.1, 5.0]\\
         PCA &No &Yes, No\\
         \hline
         & \multicolumn{2}{c}{TRPO}\\
         \hline
         Num. Iter. &300 &{200, 300}\\
         Horizon &50 &-\\
         Sim. steps per Iter. &500 &-\\
         $\delta_{KL}$ &0.1 &-\\
         Discount ($\gamma$) &0.99 &-\\
         Number of seeds &8 &-\\
         \hline
     \end{tabular}
 \end{table}%

\subsubsection{Parameters Sensitivity}
\label{apx:sensitivity}

In Figure~\ref{fig:sensitivity}, we show how the selection of the main parameters of MEPOL impacts on the learning process. To this end, we consider a set of experiments in the illustrative MountainCar domain, where we vary one parameter at a time to inspect the change in entropy index. As we can notice, the algorithm shows little sensitivity to the number of neighbors ($k$) considered in the entropy estimation. Allowing off-policy updates through an higher KL threshold $\delta$ positively impacts the learning efficiency. Furthermore, MEPOL displays a good behavior even when we limit the batch-size.

\begin{figure*}[h!]
    \begin{subfigure}[h]{0.33\textwidth}
    \centering
		\includegraphics[valign=t]{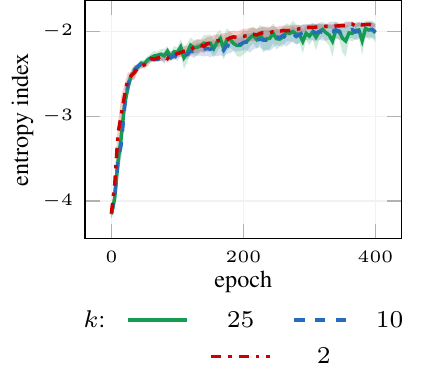}
        \caption{$k$-NN}
    \end{subfigure}
    \begin{subfigure}[h]{0.33\textwidth}
    \centering
        \includegraphics[valign=t]{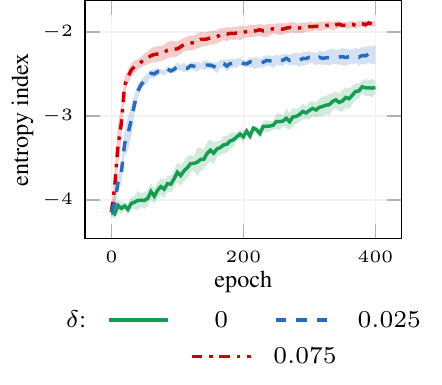}
        \caption{KL threshold}
    \end{subfigure}
    \begin{subfigure}[h]{0.33\textwidth}
    \centering
        \includegraphics[valign=t]{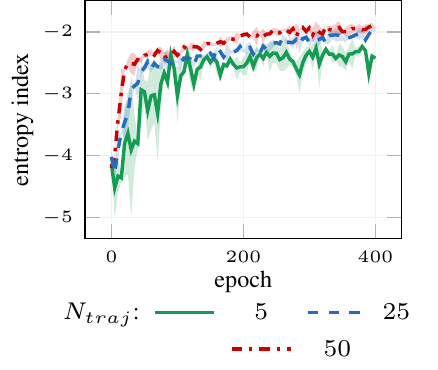}
        \caption{Batch-size}
    \end{subfigure}
    \caption{Comparison of the entropy index as a function of the learning epochs for MEPOL with different set of parameters on the MountainCar domain. (95\% c.i. over 8 runs, $T=400$ (a,b,c), $k = 4$ (b,c), $\delta = 0.05$ (a,c), $N_{traj} = 100$ (a,b)).}
    \label{fig:sensitivity}
\end{figure*}

\subsubsection{State-Visitation Heatmaps}
\label{apx:heatmaps}
In Figure~\ref{fig:grid_heatmap}, and Figure~\ref{fig:mountain_heatmap}, we report the MEPOL state-coverage evolution over GridWorld and MountainCar domains. In Figure \ref{fig:ant_heatmap}, you can see the state-coverage evolution in the Ant domain over a 12 by 12 units space, which is centered in the Ant starting position. The well-defined borders in the heatmap are due to trajectories going out of bounds. These heatmaps were created while running the experiments presented in Section~\ref{sec:exp_tae} by discretizing the continuous state-space.

\begin{figure}[H]
\centering
    \includegraphics[scale=0.4, valign=t]{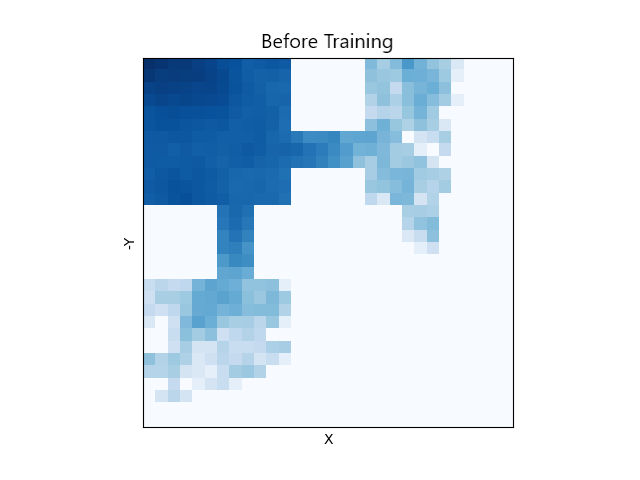}
    \includegraphics[scale=0.4, valign=t]{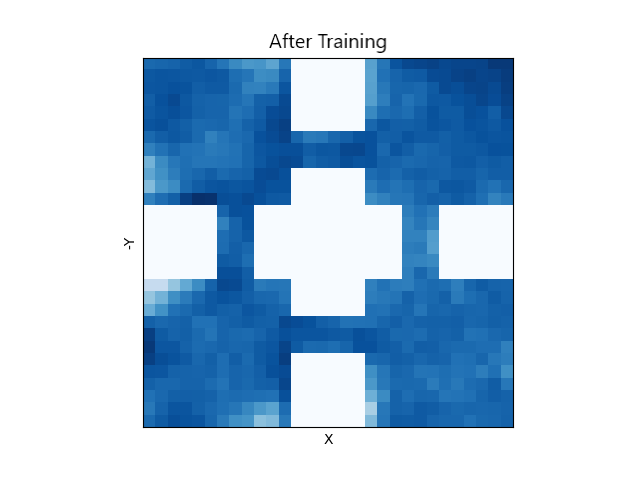}
    \caption{MEPOL log-probability state visitation evolution in the GridWorld domain created by running the policy for $N_{traj}=100$ trajectories in a time horizon of $T=1200$.}
    \label{fig:grid_heatmap}
\end{figure}

\begin{figure}[H]
\centering
    \includegraphics[scale=0.4, valign=t]{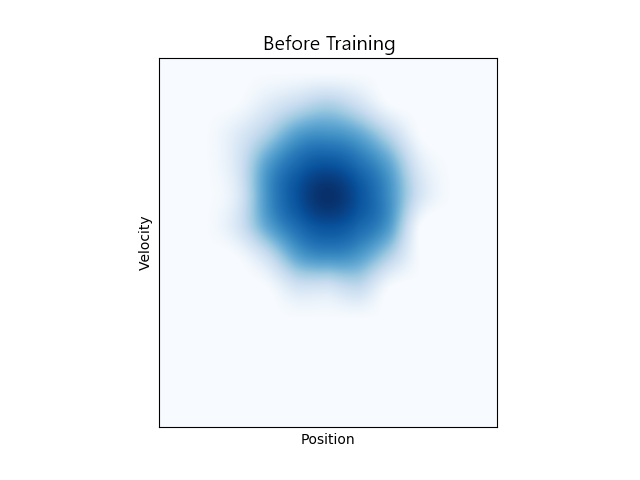}
    \includegraphics[scale=0.4, valign=t]{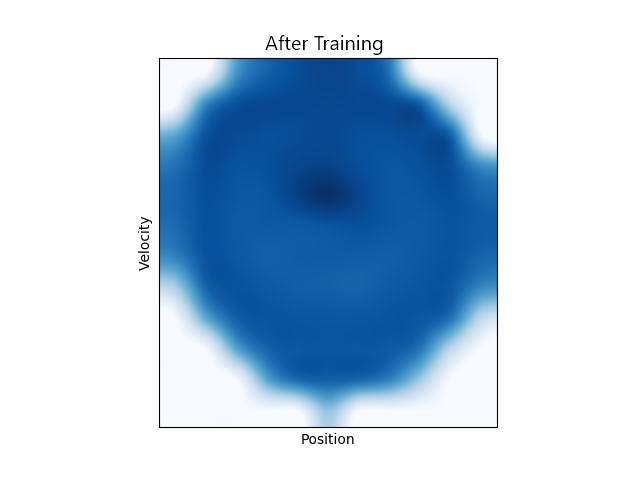}
    \vspace{0.15cm}
    \caption{MEPOL log-probability state visitation evolution in the MountainCar domain created by running the policy for $N_{traj}=100$ trajectories in a time horizon of $T=400$.}
    \label{fig:mountain_heatmap}
\end{figure}

\begin{figure}[H]
\centering
    \includegraphics[scale=0.4, valign=t]{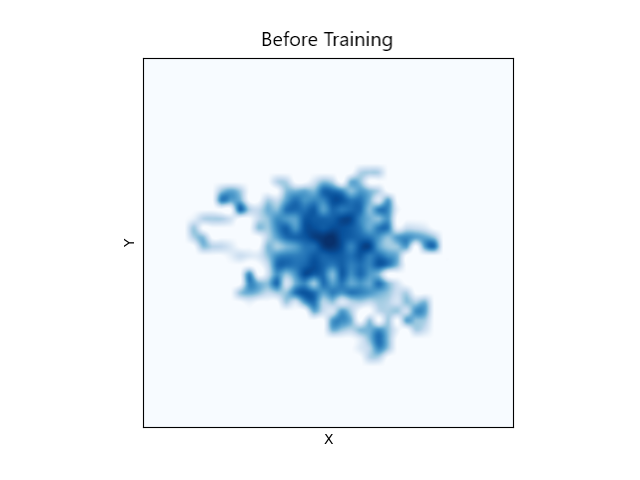}
    \includegraphics[scale=0.4, valign=t]{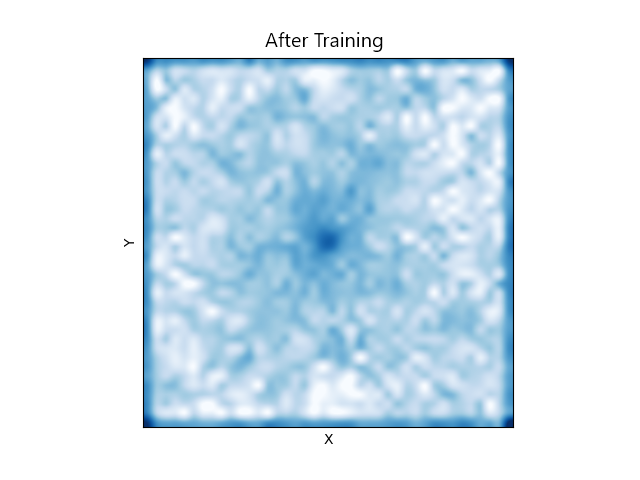}
    \vspace{0.15cm}
    \caption{MEPOL log-probability (x, y) state visitation evolution in the Ant domain created by running the policy for $N_{traj}=100$ trajectories in a time horizon of $T=500$.}
    \label{fig:ant_heatmap}
\end{figure}

\subsubsection{Discrete Entropy}
In Figure~\ref{fig:discrete_entropy_comparison}, we report the plots for the evaluation of the entropy on the 2D-discretized state-space from which we have taken the values reported in Figure~\ref{fig:discrete_entropy}.
In Ant and Humanoid, the considered state-space is the 2D, discretized, agent's position (x, y).
These plots were created while running the experiments in Section~\ref{sec:exp_tae}.

\begin{figure*}[h]
	\centering \includegraphics[scale=0.98]{plot_files/legend_illustrative.pdf}
	\vspace{-0.07cm}

    \begin{subfigure}[t]{0.32\textwidth}
    \centering
		\includegraphics[valign=t]{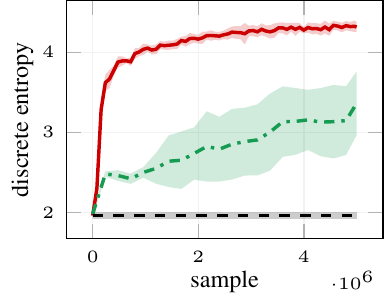}
        \caption{MountainCar}
        \label{fig:discrete_entropy_mountain_car}
    \end{subfigure}
    \begin{subfigure}[t]{0.32\textwidth}
    \centering
        \includegraphics[valign=t]{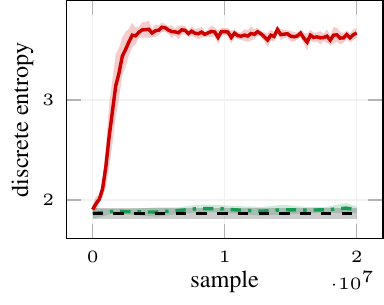}
        \caption{Ant}
        \label{fig:discrete_entropy_ant_torso}
    \end{subfigure}
    \begin{subfigure}[t]{0.32\textwidth}
    \centering
        \includegraphics[valign=t]{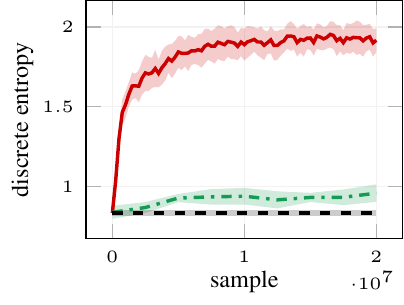}
        \caption{Humanoid}
        \label{fig:discrete_entropy_humanoid}
    \end{subfigure}
    \caption{
    Comparison of the entropy computed on the 2D-discretized state-space as a function of training samples achieved by MEPOL, MaxEnt, and a random policy in the MountainCar, Ant and Humanoid domains in the setting presented in Section \ref{sec:exp_tae} (95\% c.i. over 8 runs).
    }
    \label{fig:discrete_entropy_comparison}
\end{figure*}

\subsection{Goal-Based Reinforcement Learning}

\subsubsection{Algorithms}
We use the TRPO implementation from OpenAI's SpinningUp library \url{https://github.com/openai/spinningup}. For SAC, we adopt the codebase from \url{https://github.com/microsoft/oac-explore}. We use the original SMM codebase \url{https://github.com/RLAgent/state-marginal-matching}, which provides also an implementation of ICM, and Pseudocount.

\subsubsection{Parameters Detail}
In Table \ref{table:apx_trpo}, we report the TRPO parameters used in the tasks, following the same notation in \citep{duan2016benchmarking}. Both the basic agent and the MEPOL agent use the same parameters. In Table \ref{table:apx_sac}, we report the SAC parameters, following the notation in \citep{haarnoja2018soft}. In SMM we use 4 skills. In SMM, ICM, and Pseudocount we equally weight the extrinsic and the intrinsic components, as we haven't seen any improvement doing otherwise. Note that SMM, ICM, and Pseudocount are built on top of SAC for which we adopt the same parameters as in Table \ref{table:apx_sac}. The neural policy architectures are not reported as they are the same as in Section~\ref{sec:mepol_param}.

\begin{table}[H]
    \centering
    \caption{TRPO Parameters for Goal-Based Reinforcement Learning }
    \label{table:apx_trpo}
    \begin{tabular}[t]{lccccc}
        \hline
        &GridWorld &AntEscape &AntJump &AntNavigate &HumanoidUp\\
        \hline
        Num. Iter. &100 &500 &1000 &1000 &2000 \\
        Horizon &1200 &500 &500 &500 &2000\\
        Sim. steps per Iter. &12000 &5000 &50000 &50000 &20000 \\
        $\delta_{KL}$ &$10^{-4}$ &$10^{-2}$ &$10^{-2}$ &$10^{-2}$ &$10^{-2}$\\
        Discount ($\gamma$) &0.99 &0.99 &0.99 &0.99 &0.99\\
        Number of seeds &8 &8 &8 &8 &8\\
        \hline
    \end{tabular}
\end{table}%

\begin{table}[H]
    \centering
    \caption{SAC Parameters for Goal-Based Reinforcement Learning}
    \label{table:apx_sac}
    \begin{tabular}[t]{lccccc}
        \hline
        &GridWorld &AntEscape &AntJump &AntNavigate &HumanoidUp\\
        \hline
        Epoch &100 &500 &1000 &1000 &2000\\
        Num. Updates &12000 &5000 &5000 &5000 &6000\\
        Learning Rate &$3\cdot10^{-4}$ &$3\cdot10^{-4}$ &$3\cdot10^{-4}$ &$3\cdot10^{-4}$ &$3\cdot10^{-4}$\\
        Discount ($\gamma$) &0.99 &0.99 &0.99 &0.99 &0.99\\
        Replay buffer size ($\gamma$) &$10^6$ &$10^6$ &$10^6$ &$10^6$ &$10^6$\\
        Number of samples per mini batch &256 &256 &256 &256 &256\\
        Number of seeds &8 &8 &8 &8 &8\\
        \hline
    \end{tabular}
\end{table}%

\subsubsection{Higher-Level and Lower-Level Policies}
\label{apx:higher_lower_level}

In this section, we discuss the performance achieved by MEPOL policies when facing higher-level tasks, such as 2D navigation (Figure~\ref{fig:ant_navigate_HL}) and standing-up (Figure~\ref{fig:humanoid_up_HL}). Especially, we can see that higher-level MEPOL policies, which are trained to maximize the entropy over spatial coordinates (x-y in Ant, x-y-z in Humanoid), outperform lower-level MEPOL policies, which are trained to maximize entropy over spatial coordinates, orientation, and joint angles (as reported in Section~\ref{sec:exp_tae}). This is not surprising, since higher-level policies better match the level of abstraction of the considered tasks. However, it is worth noting that also lower-level policies achieve a remarkable initial performance, and a positive learning trend, albeit experiencing lower convergence.

\begin{figure*}[h!]
	\centering \includegraphics[scale=0.98]{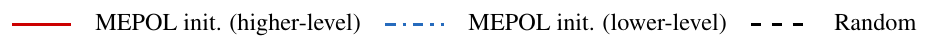}
	\vspace{-0.07cm}

    \begin{subfigure}[t]{0.32\textwidth}
    \centering
		\includegraphics[valign=t]{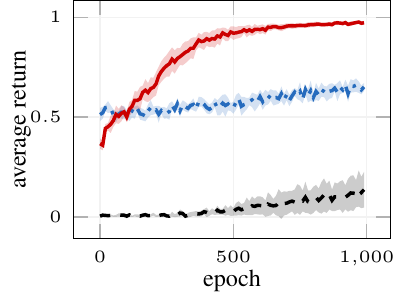}
        \caption{Ant Navigate}
        \label{fig:ant_navigate_HL}
    \end{subfigure}
    \begin{subfigure}[t]{0.32\textwidth}
    \centering
        \includegraphics[valign=t]{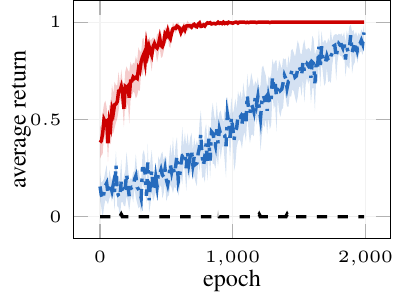}
        \caption{Humanoid Up}
        \label{fig:humanoid_up_HL}
    \end{subfigure}
    \caption{Comparison of the average return as a function of learning epochs achieved by TRPO with MEPOL initialization (higher-level, lower-level) and a Random initialization (95\% c.i. over 8 runs).}
    \label{fig:HL_comparison}
\end{figure*}

\subsection{Additional Experiments}
In this section, we present two additional experiments. First, we present another experiment in the HandReach domain in which we maximize the entropy over the 3D position of each of the fingertip, for a total of 15 dimensions, whose performance is reported in Figure \ref{fig:hand_reach15_entropy}.
Then, we present MEPOL in the AntMaze domain, in which the objective is to uniformly cover the 2D position of the agent, moving in a maze. The performance is reported, together with a visual representation of the environment, and a state-visitation evolution, in Figure~\ref{fig:ant_maze_entropy}.

\begin{figure}[ht]
\begin{center}
\centerline{\includegraphics[scale=1]{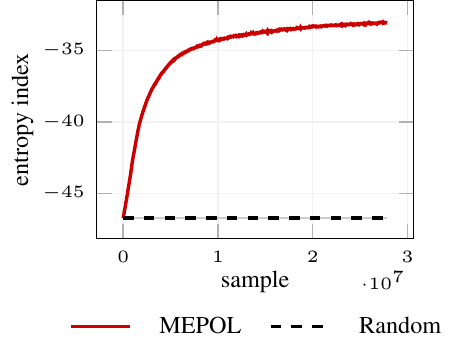}}
\caption{
Performance of the entropy index as a function of training samples achieved by MEPOL and a random policy in the additional HandReach experiment (95\% c.i. over 8 runs, $k = 4$, $T = 50$, $N_{traj}=100$, $\delta=0.05$).}
\label{fig:hand_reach15_entropy}
\end{center}
\end{figure}

\begin{figure}[ht]
  \begin{subfigure}[b]{0.5\linewidth}
    \centering
    \includegraphics[scale=0.275]{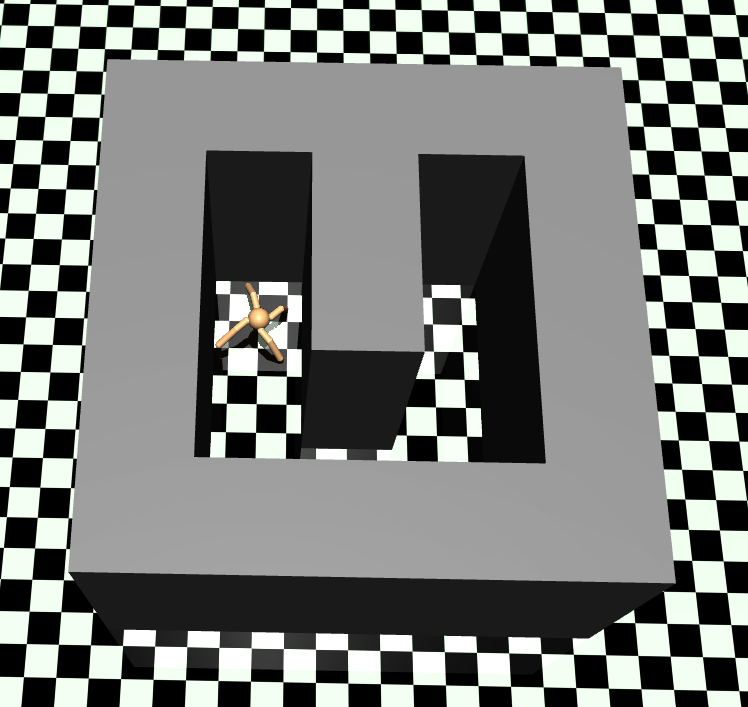}
    \caption{}
    \vspace{4ex}
  \end{subfigure}
  \begin{subfigure}[b]{0.5\linewidth}
    \centering
    \hspace{-0.5cm}\includegraphics[scale=1]{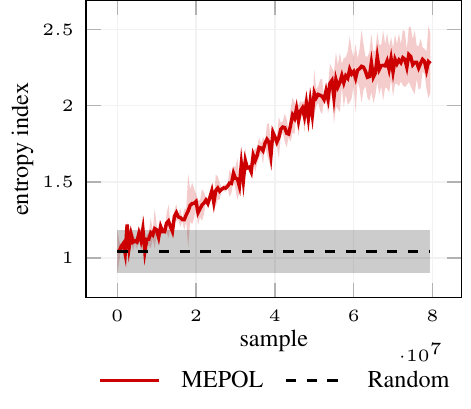}
    \caption{}
    \vspace{3.5ex}
  \end{subfigure}
  \begin{subfigure}[b]{0.5\linewidth}
    \centering
    \includegraphics[scale=0.425]{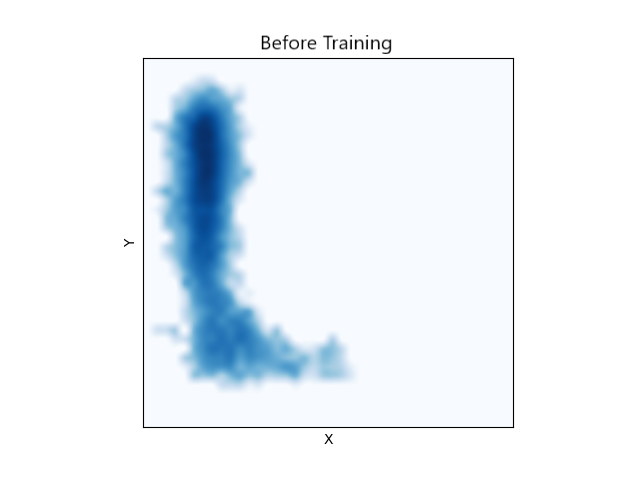}
    \caption{}
  \end{subfigure}
  \begin{subfigure}[b]{0.5\linewidth}
    \centering
    \includegraphics[scale=0.425]{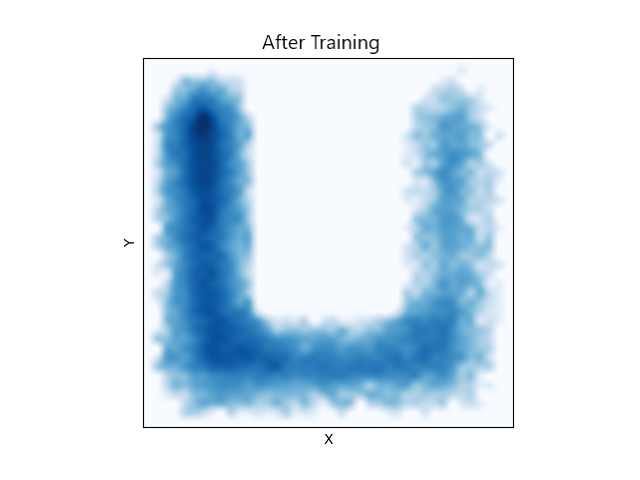}
    \caption{}
  \end{subfigure}
  \caption{
  Performance of the entropy index (b) as a function of training samples achieved by MEPOL and a random policy in the AntMaze environment (a) (95\% c.i. over 2 runs, $k = 50$, $T = 500$, $N_{traj}=100$, $\delta=0.05$) together with the log-probability state evolution before (c) and after (d) training created by running the policy for $N_{traj} = 100$ trajectories in the optimized time horizon.}
  \label{fig:ant_maze_entropy}
\end{figure}


\end{document}